\documentclass[%
amsmath,amssymb,
 aps,
 rmp,
 10pt,
 twocolumn,
 floatfix,
 superscriptaddress, 
]{revtex4-2}

\usepackage{float}

\usepackage{graphicx}
\usepackage{dcolumn}
\usepackage{bm}
\usepackage{hyperref}
\usepackage[export]{adjustbox}
\usepackage{booktabs}
\usepackage{xspace}

\usepackage{color}
\usepackage{amsmath}
\usepackage{amsthm}
\newtheorem{theorem}{Theorem}

\newtheorem{corollary}[theorem]{Corollary}
\newtheorem{lemma}[theorem]{Lemma}

\usepackage{xcolor}

\newcommand{\prob}{\mathrm{P}}

\newcommand{\vocab}{V}

\newcommand{\seq}{a}

\newcommand*{\seqbm}{\bm{a}}
\newcommand{\seqbmtitle}{\texorpdfstring{$\bm{a}$}{a}}

\newcommand{\seqs}{\Gamma}
\newcommand{\word}{w}
\newcommand{\eos}{\texttt{<END>}\xspace}

\usepackage{xr}
\newcommand{\appref}[1]{Supp. \ref{#1}}

\makeatletter
\newcommand*{\balancecolsandclearpage}{%
  \close@column@grid
  \cleardoublepage
  \twocolumngrid
}
\makeatother

\graphicspath{{img/}{suppl-img/}}

\begin{document}

\preprint{APS/123-QED}

\title{Correlation Dimension of Natural Language in a Statistical Manifold}

 \author{Xin Du}%
 \email{duxin.ac@gmail.com}
\affiliation{%
 Waseda Research Institute for Science and Engineering,
 Waseda University.}

\author{Kumiko Tanaka-Ishii}
 \email{kumiko@waseda.jp}
 \affiliation{Department of Computer Science and Engineering, School of Fundamental Science and Engineering, Waseda University.}

\date{\today}

\begin{abstract}
The correlation dimension of natural language is measured by applying the Grassberger-Procaccia algorithm to high-dimensional sequences produced by a large-scale language model. This method, previously studied only in a Euclidean space, is reformulated in a statistical manifold via the Fisher-Rao distance. Language exhibits a multifractal, with global self-similarity and a universal dimension around 6.5, which is smaller than those of simple discrete random sequences and larger than that of a Barab\'asi-Albert process. Long memory is the key to producing self-similarity. Our method is applicable to any probabilistic model of real-world discrete sequences, and we show an application to music data.
\end{abstract}

\maketitle

\onecolumngrid
\vskip-3em
\tableofcontents
\twocolumngrid

\balancecolsandclearpage
\section{Introduction}
The correlation dimension of \citet{grassberger1983characterization} quantifies the degree of recurrence in a system's evolution and has been applied to examine the characteristics of sequential data, such as the trajectories of strange attractors \citep{grassberger1983characterization}, random processes \citep{osborne1989finite}, and sequences sampled from complex networks \citep{lacasa2013correlation}.

In this letter, we report the correlation dimension of natural language by regarding texts as the trajectories of a language dynamical system. In contrast to the long-memory quality of natural language as reported in \citep{li1989mutual, altmann2012, plosone16}, the correlation dimension of natural language has barely been studied because of its high dimensionality and discrete nature. An exceptional previous work, to the best of our knowledge, was that of \citet{doxas2010dimensionality}, who measured the correlation dimension of language in terms of a set of paragraphs. Every paragraph was represented as a vector, with each dimension being the logarithm of a word's frequency. The distance between two paragraphs was measured as the Euclidean distance. Such a representation has also been used for measuring other scaling factors of language \citep{acl18, jpc18, ausloos2012measuring}. However, without a rigorous definition of language as a dynamical system, the correlation dimension is difficult to interpret, and its value may easily depend on the setting. For example, the dimension would vary greatly between handling word frequencies logarithmically and nonlogarithmically.

Today, language representation has become elaborate by incorporating semantic ambiguity and long context. {\em Large language models} (LLMs) \citep{radford2019language,openai2023gpt4,touvron2023llama,yi} such as ChatGPT generate texts that are hardly distinguishable from human-generated texts. The generation process is autoregressive, which naturally associates a dynamical system. Such state-of-the-art (SOTA) models (i.e., the GPT series, including GPT-4 \citep{openai2023gpt4}, Llama-2 \citep{touvron2023llama}, and ``Yi'' \citep{yi}) have opened a new possibility of studying the physical nature of language as a complex dynamical system. Furthermore, exploration of the fractal dimension of language offers a novel approach to examine the underlying structures of pretrained neural networks, thus shedding light on the intricate ways they mirror human intelligence.

These new systems, however, are not defined in a Euclidean space and thus require reformulation of the state space and the metric between states. Because a neural model assumes a probability space, the analysis method that was originally defined in a Euclidean space must be accommodated in a space of probability distributions, and the distance metric must be statistical. Specifically, we consider a statistical manifold \citep{rao1992information, amari2012differential} whose metric is the Fisher information metric. Hence, this letter proposes a rigorous formalization to analyze the universal properties of these GPT models, thus representing language as an original dynamical system. Although we report results mainly for language, given the impact of ChatGPT, our formalization applies to any other GPT neural models for real-world sequences, such as DNA, music, programming sources, and finance data. To demonstrate this possibility, we show an application to music.

\section{Method}
Let $(S, d)$ be a metric space and $[x_1,x_2,\cdots,x_N]$ be a point sequence, where $x_t\in S$ for $t=1,\cdots,N$. The Grassberger-Procaccia algorithm \citep{grassberger1983characterization} (GP in the following) defines the correlation dimension of this point sequence in terms of an exponent $\nu$ via the growth of the correlation integral $C(\varepsilon)$, as follows:
\begin{equation}
 C(\varepsilon) \sim \varepsilon^\nu~~~\text{as}~\varepsilon\to 0,
\end{equation}
where
\begin{equation}
 C(\varepsilon) = \lim_{N\to\infty} \frac{1}{N^2} \sum_{1\leq t,s\leq
 N} \#\Bigl\{(t,s): d(x_t, x_s) < \varepsilon\Bigr\},
 \label{eq:corrintegral}
\end{equation}
$\#$ denotes a set's size, and $d$ is the distance metric. In the original GP, the sequence lies in a Euclidean space and $d$ is the Euclidean distance. For an ergodic sequence, the correlation dimension suggests the values of other fractal dimensions such as the Haussdorf dimension \citep{pesin1993rigorous}. For example, the H\'enon map has $\nu=1.21\pm 0.01$ \citep{grassberger1983characterization}, which is close to its Hausdorff dimension of $1.261\pm 0.003$ \citep{russell1980dimension}. GP can be generalized to apply to a sequence in a more general smooth manifold \citep{pesin1993rigorous}.

\begin{figure*}[tbp]
 \centering
 \begin{minipage}[t]{0.42\linewidth}
 \includegraphics[width=\linewidth]{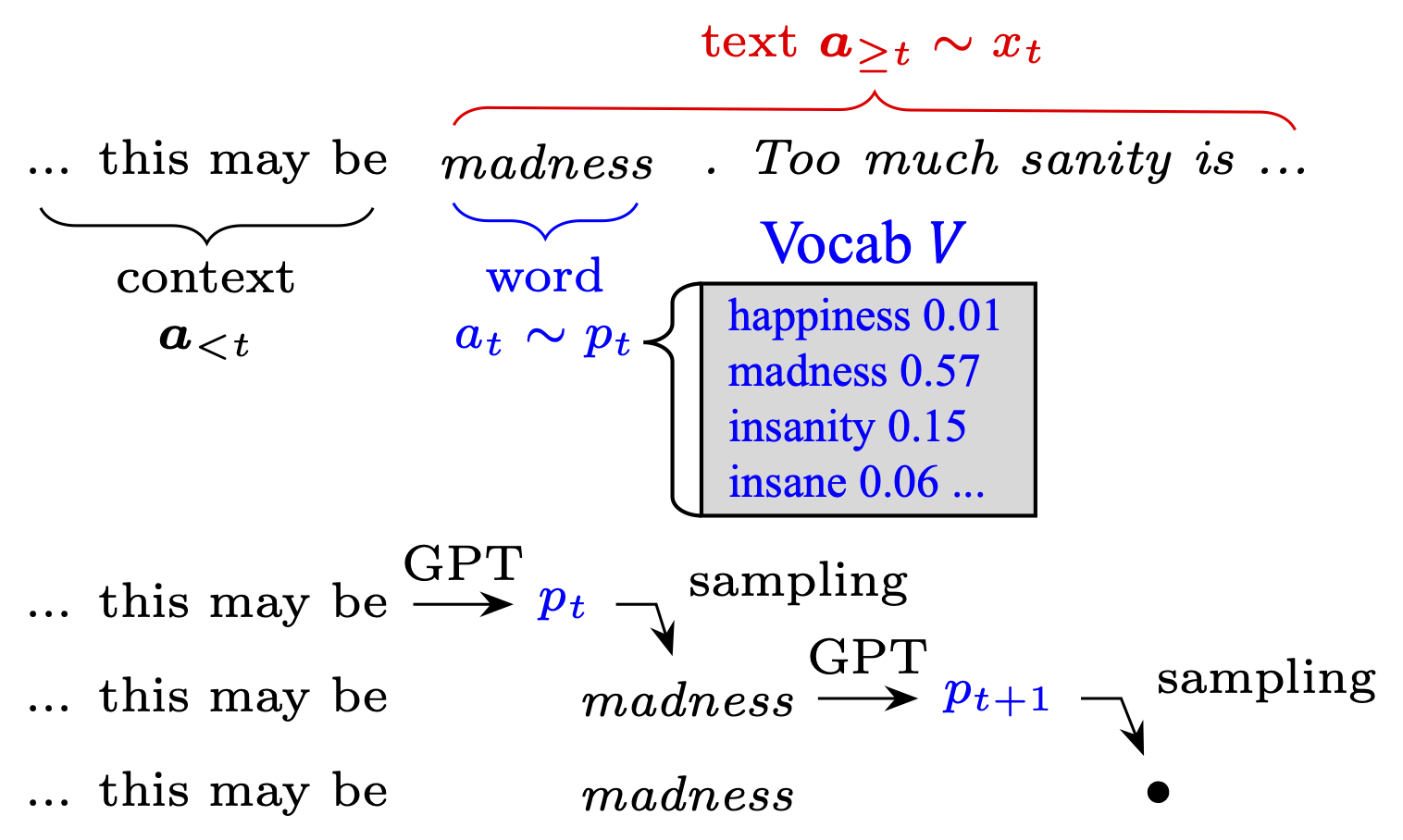}
 \\(a)
 \end{minipage}
 \hspace{0.02\linewidth}
 \begin{minipage}[t]{0.50\linewidth}
 \includegraphics[width=\linewidth]{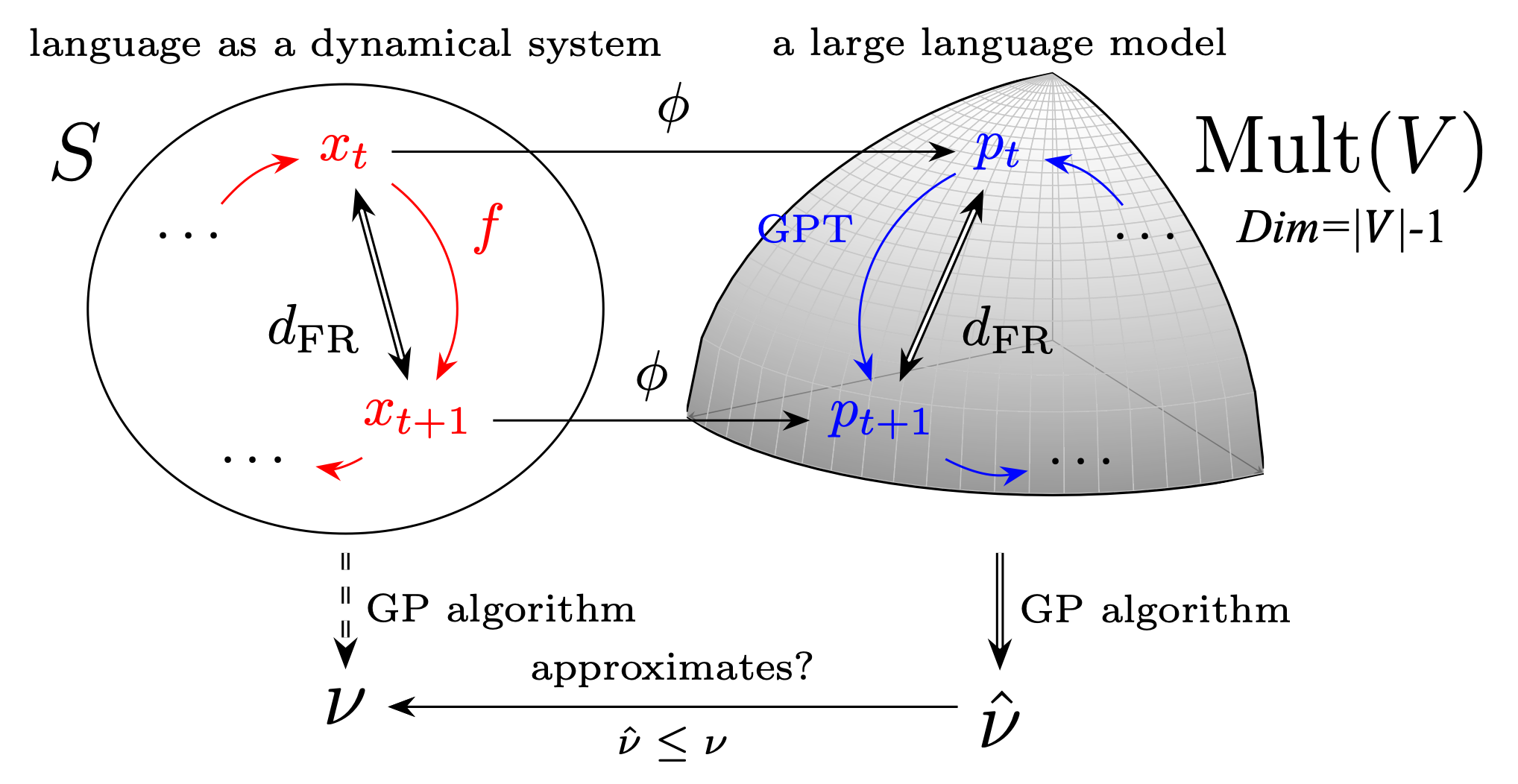}
 \\(b)
 \end{minipage}
\caption{Our model of language as a stochastic dynamical system. (a) The difference between the system state $x_t$ and the next-word probability distribution $p_t$. (b) $\{p_t\}$ (where $p_t\in \text{Mult}(\vocab)$) as the image of $\{x_t\}$ (where $x_t\in S$) through the marginalization mapping $\phi$ in Formula (\ref{eq:phi}). In this study, we use $\hat{\nu}$ to approximate $\nu$. }

 \label{fig:langsys}
\end{figure*}

In our study, we examine natural language through this correlation dimension. Thus far, language texts have typically been considered in a Euclidean space. However, recent large language models have shown unprecedented performance in the form of an autoregressive system, which is defined in a probability space. Hence, we are motivated to measure the correlation dimension in a statistical manifold.

We consider a language dynamical system $\{x_t\}$ that develops word by word: $f: x_t\mapsto x_{t+1}$. Let $\vocab$ represent a vocabulary that comprises all unique words. A sequence of words, $\seqbm=[\seq_1,\seq_2,\cdots,\seq_t,\cdots]$, where $\seq_t\in\vocab$, is associated with a sequence of system states, $[x_1,x_2,\cdots,x_t,\cdots]$. As demonstrated in Figure \ref{fig:langsys}(a) at the top, we define each state $x_t$ as a probability distribution over the set $\seqs$ of all word sequences. $x_t$ measures the probability of any text to occur as $\seqbm_{\geq t}=[\seq_t,\seq_{t+1},\cdots]$, following a {\em context} $\seqbm_{<t}=[\seq_1,\cdots,\seq_{t-1}]$. Furthermore, we consider the next-word probability distribution $p_t$ over the vocabulary $\vocab$. $x_t$ and $p_t$ are formally defined as follows:
\begin{alignat}{2}
 &x_t(\seqbm_{\geq t}) = \prob(\seqbm_{\geq t} \mid \seqbm_{<t}) ~~~~~~ &&\forall \seqbm_{\geq t}\in\seqs,
 \label{eq:xt} \\
 &p_t(\word) = \prob(\seq_{t}=\word \mid \seqbm_{<t}) ~~~~~~ &&\forall \word\in\vocab.
 \label{eq:pt}
\end{alignat}
$p_t$ can be represented as the image of $x_t$ by a mapping $\phi$:
\begin{equation}
 p_t = \phi (x_t). \label{eq:phi}
\end{equation}
Here, $\phi$ is the marginalization across $\seqs$ and is linear with respect to a mixture of distributions, as explained in \appref{sec:linearity}.

Hence, a language state $x_t$ is represented as a probability function instead of a point in a Euclidean space. The correlation dimension $\nu$ can be defined for the sequence $\{x_t\}$ as long as the distance metric $d$ in Formula (\ref{eq:corrintegral}) is specified between any pair of states $x_t$ and $x_s$. However, direct acquisition of $d(x_t,x_s)$ is nontrivial because $\{x_t\}$ as a language is unobservable. One new alternative path today is to represent $x_t$ via $p_t$, where $p_t$ is produced by a large language (especially a GPT-like) model (LLM). We denote the correlation dimension of the sequence $\{p_t\}$ as $\hat{\nu}$. Our approach is summarized in Figure \ref{fig:langsys}(b) at the bottom. \appref{sec:gpt} provides a brief introduction to GPT-like LLMs.

Theoretically, $\hat{\nu}=\nu$ when the sequence of words is generated by a Markov process. We prove this in \appref{sec:markov}. Natural language exhibits the Markov property to a certain extent, but strictly speaking, it violates the property. This phenomenon has been studied in terms of long memory \citep{li1989mutual,altmann2009, altmann2012, plosone16}, as mentioned in the Introduction. Therefore, the $\hat{\nu}$ acquired from ${p_t}$ will remain an approximation of $\nu$. In general, $\hat{\nu}\leq \nu$ holds \cite{peitgen1992chaos} and $\hat{\nu}$ thus constitutes a lower bound of $\nu$.

The distance metric $d$ in Formula (\ref{eq:corrintegral}) is chosen as the Fisher-Rao distance, defined as the geodesic distance on a statistical manifold generated by Fisher information \citep{amari2012differential}. When $\{p_t\}$ is presumed to follow a multinoulli distribution (over the vocabulary $\vocab$), the statistical manifold is the space of all multinoulli distributions over $\vocab$, denoted as $\text{Mult}(\vocab)$, as shown at the top right in Figure \ref{fig:langsys}(b). $\text{Mult}(\vocab)$ has a (topological) dimension of $|\vocab|-1$ and is isometric to the positive orthant of a hypersphere. The Fisher-Rao distance is analytically equal to twice the Bhattacharyya angle, as follows:
\begin{equation}
\begin{aligned}
 d_\text{FR}(p_t, p_s) = 2 \arccos \left(
 \sum_{\word\in\vocab} \sqrt{p_t(\word) p_s(\word)}
 \right) \\
 ~~~~\text{for}~t,s=1,2,\cdots,N.
 \label{eq:fisher-rao}
\end{aligned}
\end{equation}
This statistical manifold is a Riemannian manifold of constant curvature (as it constitutes a part of a hypersphere), sharing many favorable topological properties with Euclidean spaces. Particularly, the Marstrand projection theorems \citep{marstrand1954some,falconer2004fractal} for Euclidean spaces, which state that linear mappings almost surely preserve a set's Hausdorff dimension, can be generalized to such Riemannian manifolds. Recently, \citet{balogh2016dimensions} proved Marstrand-like theorems for sets on a 2-sphere. Because the mapping $\phi: x_t\mapsto p_t$ is linear, as mentioned before and proved in \appref{sec:linearity}, these theorems could be generalized to suggest the equality $\nu=\hat{\nu}$. This possible generalization goes beyond this letter's scope; even if it were true, Marstrand-like theorems do not guarantee a specific linear mapping (i.e., $\phi$) to be dimension-preserving. Nevertheless, these theorems motivate our proposal to analyze $\nu$ via its lower bound $\hat{\nu}$.

The calculation of distances over $N$ timesteps takes $O(|\vocab|\cdot N^2)$ time, with a vocabulary size $|\vocab|$ around $10^4$. This computational cost can be reduced to $O(M\cdot N^2)$ through dimension reduction from $\{p_t\}$ to $\{q_t\}$, without altering the estimated correlation dimension $\hat{\nu}$, where $M\ll |\vocab|$ is the new, smaller dimensionality. For $t=1,\cdots,N$, the dimension-reduction projection transforms $p_t$ to $q_t$ as follows:
 \begin{equation}
 q_t(m) = \sum_{\word\in\Phi^{-1}(\{m\})} p_t(\word),~~~\forall m=1,\cdots,M.
 \label{eq:groupprob}
 \end{equation}
Here, $\Phi$ is determined via the modulo function: $\Phi(\word) =\text{index}(\word)~\text{mod}~M$, where $\text{index}(\word)$ indicates a word's index in the vocabulary. Essentially, we ``randomly'' group words from the extensive vocabulary $\vocab$ in a smaller set $\{1,\cdots,M\}$ and estimate $\hat{\nu}$ according to this condensed vocabulary. We empirically validated this method, which is rooted in Marstand's projection theorem, as detailed in \appref{sec:dimreduce}. Specifically, dimensionality reduction from approximately 50,000 to 1,000 retained the consistency of estimating $\hat{\nu}$ and achieved up to 50X faster computation.

\section{Results}
\begin{figure*}[tb]
\centering
\begin{minipage}{1.0\linewidth}
\includegraphics[width=\linewidth]{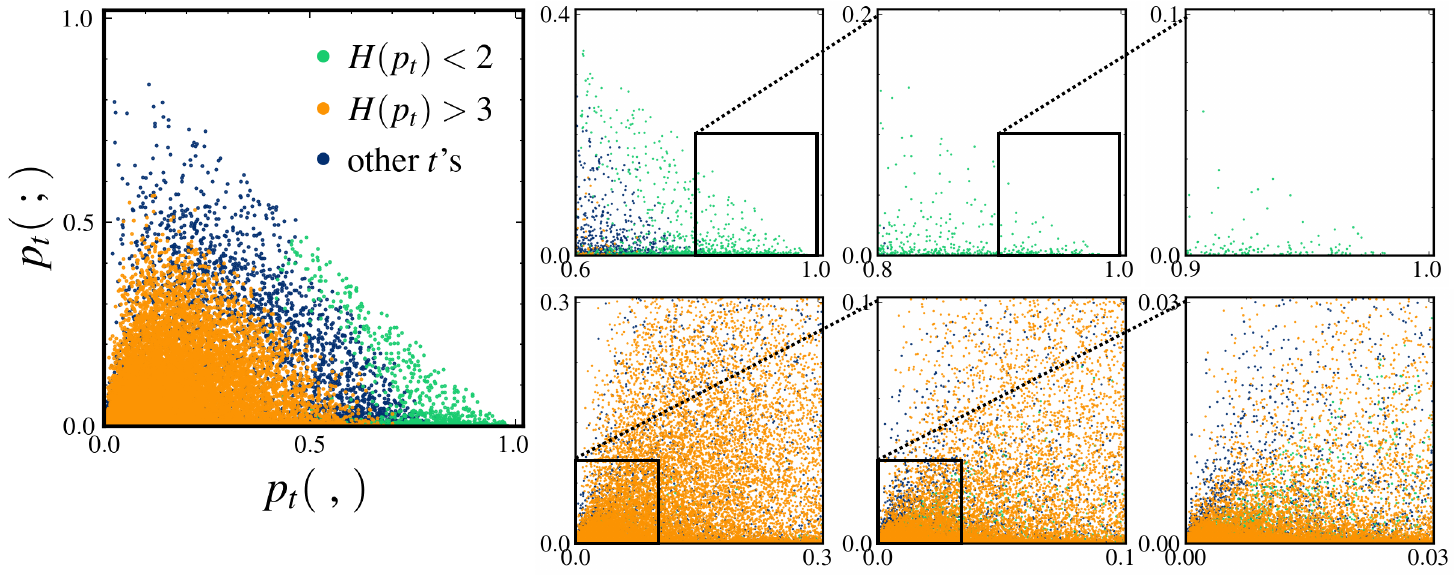}
\end{minipage}
\caption{Sequence of distributions $p_t$ underlying the words in {\em Don Quixote}, as visualized for words ``,'' (comma) and ``;'' (semicolon). Each point represents one timestep. The green points represents timesteps at which $p_t(\text{``,''})$ dominates and the Shannon entropy $H(p_t) < 2.0$, whereas the orange points correspond to high-entropy states with $H(p_t) > 3.0$. Self-similar patterns are observed in both the green and orange regions. \label{fig:p}}
\end{figure*}

Before showing the correlation dimension, we examine language's inherent self-similarity. Figure \ref{fig:p} includes a plot showing the probability $p_t$ of encountering ``,'' (commas) and ``;'' semicolons over $t=1,2,\cdots,N$ in an English translation of {\em Don Quixote} by Miguel de Cervantes from Project Gutenberg \footnote{\url{https://www.gutenberg.org/ebooks/996}}. These punctuation marks, chosen for their high frequency, illustrate the role of semantic ambiguity. Each $p_t$ represents a point in $\text{Mult}(\vocab)$, a probability vector of the next-word occurrence, estimated using GPT2-\texttt{xl} \citep{radford2019language}. The figure maps these points, varying with input context $\seqbm_{<t}$, and classifies them by Shannon entropy $H(p_t)$, revealing self-similarity in both low- and high-entropy regions through magnified views at different scales. Nevertheless, a thorough assessment of this self-similarity necessitates examining the high-dimensional space of $\text{Mult}(\vocab)$, beyond the limits of a two-dimensional display that cannot represent correlation dimensions above 2.

We conjecture that the trajectory has two kinds of fractals: local
and global. The local fractals, potentially arising from simple word
distributions across contexts akin to those in topic models like LDA
\citep{blei2003latent}, are evident in low-entropy areas where single
words predominate. In \appref{sec:local-dirichlet}, we show that even
i.i.d. samples from a Dirichlet distribution (a commonly assumed
prior for multinoulli distributions) can reproduce the local fractal
seen in Figure \ref{fig:p}. The local kind's occurrence could
be related to the finding in \citet{doxas2010dimensionality} that
topic models can reproduce self-similar patterns. However, the
local kind is not especially concerned in this letter because it
characterizes single words and hence does not reveal the nature of
the original system $\{x_t\}$.

In this letter, we are mainly interested in the correlation dimension of the global phenomenon. Unlike the local kind, the global fractals represent high-entropy regions that are governed by the trajectory's global development. Hence, we consider points in the higher-entropy region, as filtered by a parameter $\eta$:
\begin{equation}
 \max_{\word\in \vocab} p_t(\word) < \eta.
 \label{eq:maxprob}
\end{equation}

\begin{figure*}[tb]
 \centering
 \begin{minipage}{0.3\linewidth}
 \includegraphics[width=\linewidth]{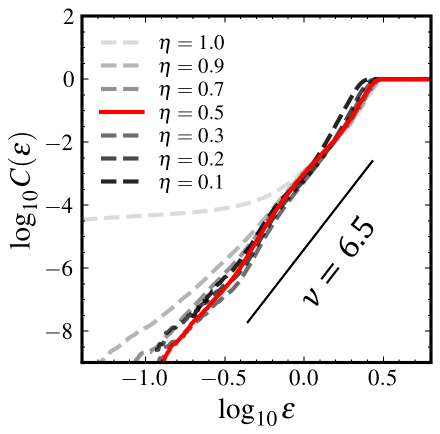}
 (a)
 \end{minipage}
 \begin{minipage}{0.3\linewidth}
 \includegraphics[width=\linewidth]{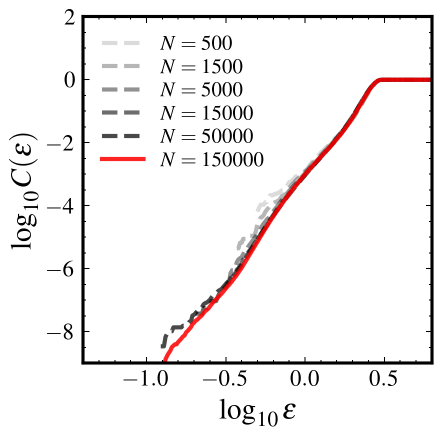}
 (b)
 \end{minipage}
 \begin{minipage}{0.3\linewidth}
 \includegraphics[width=\linewidth]{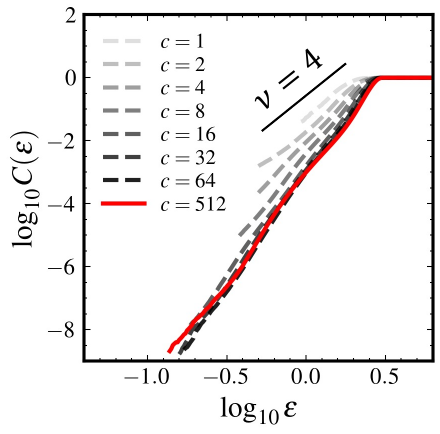}
 (c)
 \end{minipage}
\caption{Correlation integral curves as defined by Formula (\ref{eq:corrintegral}) and estimated with GPT2-\texttt{xl} with respect to (a) the maximum-probability threshold $\eta$ in Formula (\ref{eq:maxprob}), (b) the sequence length $N$, and (c) the context length $c$ in Formula (\ref{eq:ctxlen}). \label{fig:donquixote}}
\end{figure*}

Figure \ref{fig:donquixote}(a) shows the correlation integral from Formula (\ref{eq:corrintegral}) with respect to $\varepsilon$ for {\em Don Quixote} in terms of different probability thresholds $\eta$ in Formula (\ref{eq:maxprob}). As $\eta$ decreases to 0.5 (red curve), the linear region becomes visible across all scales, and the correlation dimension (given by the slope) converges to $\hat{\nu}=6.42$. In contrast, the curve for $\eta=1.0$ (i.e., when no timesteps are excluded) shows great deviation from the other curves, especially at smaller $\varepsilon$ values, producing a local correlation dimension that drops below 2.0. Hence, unless mentioned otherwise, $\eta=0.5$ in this letter. For $\eta=0.5$, Figure \ref{fig:donquixote}(a) shows a long span across more than six orders of magnitude, from $10^{-1}$ to $10^{-8}$ on the vertical axis.

Figure \ref{fig:donquixote}(b) characterizes the effect of $N$, the length of the text used to estimate the correlation dimension. The longest text fragment had 150,000 words and is indicated by the red curve. Convergence is visible for all $N$, starting from $N=500$. Unless mentioned otherwise, $N=150,000$ here.

We also investigated the effect of the context length, denoted as $c$. Ideally, an LLM estimates the distribution $p_t$ by using the whole text $[\seq_1,\cdots,\seq_{t-1}]$ before timestep $t$ as the context, but in practice, a maximum context length $c$ is often set; that is,
\begin{align}
 p_t^{(c)}(\word) &= \prob
 (\seq_t=\word\mid \seq_{t-c}, \seq_{t-c+1}, \cdots, \seq_{t-1}) \nonumber \\
 &\approx p_t(\word) ~~~~\forall \word\in\vocab.
 \label{eq:ctxlen}
\end{align}
Unless mentioned otherwise, all results in this letter were obtained with $c=512$.

Figure \ref{fig:donquixote}(c) shows the correlation dimension with
values of $c$ as small as 1 (i.e., a Markov model). For context
lengths above 32, a clear linear scaling phenomenon is observed
across all scales, which resembles the case of $c=512$. As $c$
decreases, the linear-scaling region becomes narrower and the
self-similarity becomes less evident. Dependency of the
correlation dimension on $c$ is seen only for the global fractal,
whereas the dimension is consistent across $c$ values for the local
fractals, as detailed in \appref{sec:local-ctxlen}.

This difference in the behavior of local and global fractals suggests
a fundamental difference between these two kinds. The local fractal
does not depend on $c$, whereas the global fractal requires large $c$
to appear. While the local fractal may stem from mixed word-frequency
distributions in topic models, as observed by
\citet{doxas2010dimensionality} and mentioned above, the global
fractal is due to long memory and was anticipated in the literature
\citep{li1989mutual, altmann2012, plosone16}. Although
self-similarity and long memory have often been studied separately
and were even conjectured as different aspects of a scale-invariant
process \citep{abry2003self}, they show interesting coordination for
natural language. More results on a larger dataset are provided in
\appref{sec:ctxlen}.

To further investigate the properties of natural language, we conducted a larger-scale analysis of long texts, which were divided into two groups: books in multiple languages and English articles in multiple genres, as detailed in \appref{sec:data}. The first group included 144 single-author books from Project Gutenberg and Aozora Bunko, comprising 80 in English, 32 in Chinese, 16 in German, and 16 in Japanese. The second group included 342 long English texts from different sources. We obtained all the results in this large-scale analysis by applying the dimension-reduction method given in Formula (\ref{eq:groupprob}).

\begin{figure*}[tb]
 \centering
 \begin{minipage}[t]{0.85\linewidth}
 \begin{minipage}[t]{0.43\linewidth}
 \includegraphics[width=\linewidth,valign=t]{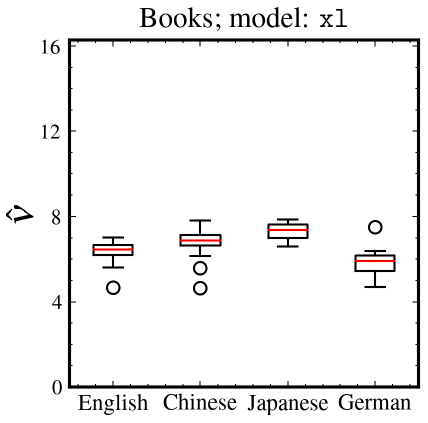}
 (a)
 \end{minipage}
 \begin{minipage}[t]{0.54\linewidth}
 \includegraphics[width=\linewidth,valign=t]{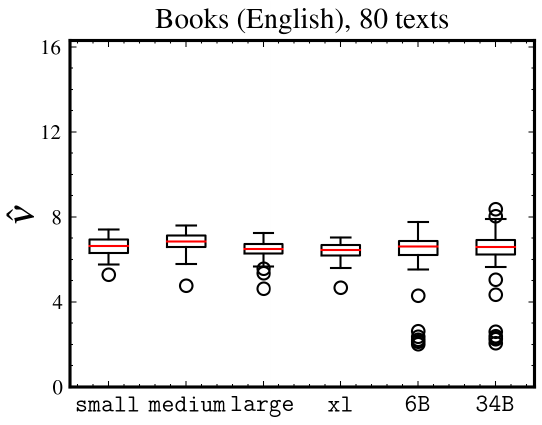}
 (b)
 \end{minipage}
 \begin{minipage}[t]{0.95\linewidth}
 \includegraphics[width=\linewidth,valign=t]{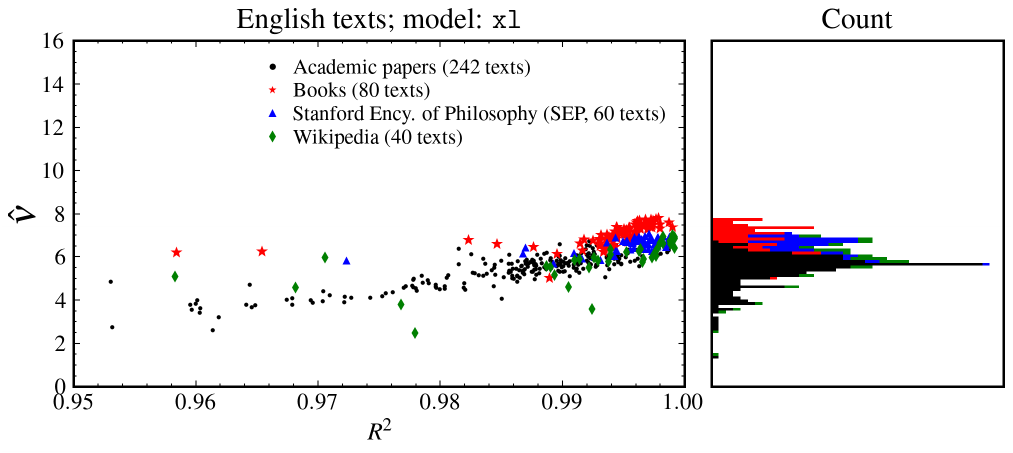}
 (c)
 \end{minipage}
 \begin{minipage}[t]{0.30\linewidth}
 \includegraphics[width=\linewidth,valign=t]{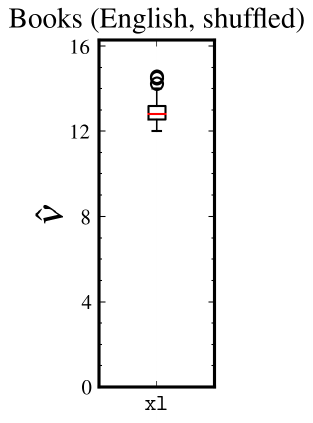}
 (d)
 \end{minipage}
 \begin{minipage}[t]{0.25\linewidth}
 \includegraphics[width=\linewidth,valign=t]{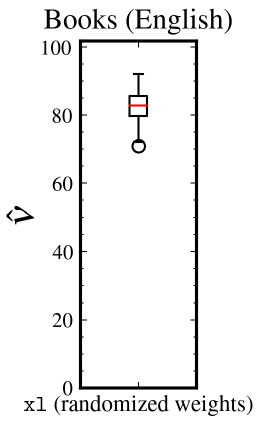}
 (e) \end{minipage}
 \end{minipage}
 \\ \vskip 1em 
\caption{Correlation dimensions of (a) all books grouped by language, as estimated using GPT2-\texttt{xl}; (b) English books as estimated using GPT with different model sizes (GPT2 from \texttt{small} to \texttt{xl} and the Yi model for \texttt{6b} and \texttt{34b}); (c) English texts from various sources with the $R^2$ scores (horizontal axis) of their linear fits to the correlation integral curves; (d) shuffled English books evaluated with GPT2-\texttt{xl}; and (e) English books evaluated with weight-randomized GPT2-\texttt{xl}.}
\label{fig:corrdim}
\end{figure*}

Figures \ref{fig:corrdim} (a) and (b) show the large-scale results on the books for the correlation dimension $\hat{\nu}$ with respect to (a) different languages and (b) various model sizes. The former results (a) were produced using the GPT2 model of size \texttt{xl} (denoting ``extra-large''), with $\approx 10^9$ parameters. For the latter results (b), we tested models of different sizes from \texttt{small} ($\approx 10^6$ parameters) to \texttt{34B} ($3.4\times 10^{10}$). For the sizes up to \texttt{xl}, we used the GPT2 model; for \texttt{6B} and \texttt{34B}, we used the Yi model \cite{yi}, which offers the SOTA capability in English among all publicly available LLMs. For all tested model sizes, the average correlation dimension remains constant. Outliers occur more frequently for the two Yi models (\texttt{6B} and \texttt{34B}), which was possibly due to those models' use of a lower numerical precision (16-bit floating-point numbers).

Hence, for all languages, an average correlation dimension of around $\hat{\nu}=6.5$ is observed: $6.39\pm 0.40$ for English, $6.81\pm 0.58$ for Chinese, $7.30\pm 0.41$ for Japanese, and $5.84\pm 0.70$ for German ($\pm$ indicates the standard deviation). These results suggest the possible existence of a common dimension for natural language, with a lower bound of 6.5 under our settings.

Figure \ref{fig:corrdim}(c) shows the correlation dimension (vertical axis) for English texts in four genres: books, academic papers \citep{Kershaw2020ElsevierOC}, the Stanford Encyclopedia of Philosophy (SEP) \footnote{\url{https://plato.stanford.edu/}}, and Wikipedia webpages. For each text, the horizontal axis indicates the coefficient of determination, $R^2$, for the correlation integral curve's linear fit. A larger $R^2$ value (maximum 1) implies more significant self-similarity in a text. The right side of (c) shows the distribution of the dimension values grouped by genre.

As seen in the figure, most texts have a correlation dimension around 6, especially those estimated with high $R^2$ scores. The SEP texts (blue) have the most concentrated range of dimensions, at $6.57\pm 0.32$ with $R^2>0.99$ for over 90\% of the texts. In contrast, the academic papers (black) show the most scattered distribution of the correlation dimension. This is deemed natural, as the SEP texts have the highest quality, whereas the academic papers include irregular notations such as chemical and mathematical formulas, which obscure a text's self-similarity.

The universal correlation dimension value, $\nu\approx 6.5$, can be understood through the lens of the ``information dimension'' \citep{farmer1982information}, which coincides with $\nu$ under ergodic conditions \citep{pesin1993rigorous}. The information dimension reflects how information, or the log count of unique contexts, scales with the statistical manifold's resolution. Contexts are deemed the same if their $p_t$ values are indistinguishably close within a certain threshold. Essentially, doubling the resolution would reveal about $2^{6.5}\approx 90$ times more distinct contexts that were previously considered identical. Therefore, $\nu$ quantifies the average ``redundancy'' in the diversity of texts conveying similar messages.

We also compared several theoretical random processes. As analyzed using a GPT2-\texttt{xl} model and shown in Figure \ref{fig:corrdim}(d), shuffled word sequences exhibited an average correlation dimension of 13.0, indicating inherent self-similarity despite the shuffling. As seen in Figure \ref{fig:corrdim}(e), randomization of the GPT2-\texttt{xl} model's weights significantly increased the correlation dimensions to an average of 80. This result suggests purely random outputs, unlike text shuffling, which retains some linguistic structures, like a bag-of-words approach.

Analyses of additional random processes, as detailed in
\appref{sec:randomprocs}, showed that a uniform white-noise process
on the statistical manifold $S$ yielded a correlation dimension over
100. Symmetric Dirichlet distributions in high-entropy regions
consistently produced dimensions above 10. Conversely,
Barab\'asi-Albert (BA) networks \citep{barabasi1999emergence}, which
are special cases of a Simon process, demonstrated a correlation
dimension of $2.00\pm 0.003$, and a fractal variant
\citep{rak2020fractional} produced $2\sim 3.5$. In terms of
complexity via the correlation dimension, this places natural
language above BA networks but below white noise.

In \appref{sec:euclid}, we further investigate the relationship between the statistical manifold and conventional Euclidean spaces with respect to the correlation dimension. For BA models, the dimension remains the same whether measured in a Euclidean space or the manifold, thus emphasizing the comparability. However, language data reveals a different story: Euclidean metrics yield compromised linearity in comparison to Fisher-Rao metrics, thus underscoring that the Fisher-Rao distance more accurately captures language's inherent self-similarity.

Recently, LLMs have also been developed for processing data beyond natural language, and one successful example is for acoustic waves compressed into discrete sequences \citep{copet2023simple}. To demonstrate the applicability of our analysis, we used the \texttt{GTZAN} dataset \citep{tzanetakis2002musical}, which comprises 1000 recorded music pieces categorized in 10 genres. Briefly, we observed clear self-similarity in the compressed music data. The correlation dimension was found to depend on the genre: classical music showed the smallest dimension at $5.44\pm 1.13$, much smaller than the dimensions for metal music at $7.27\pm 0.96$ and rock music at $7.42\pm 0.87$. None of the music genres showed a correlation dimension as large as that of white noise, as mentioned previously, even though the analysis was based on recorded data. The details of this analysis are given in \appref{sec:musicgen}.

In closing, we recognize this study's limitation of viewing text as a dynamical system akin to the GPT model, which overlooks the potential of representing words as leaf nodes in a syntactic tree, as suggested by generative and context-free grammars (CFGs) \citep{chomsky2014aspects}. Although promising, that complex linguistic framework exceeds our current scope, and we expect to explore it in the future.

\begin{acknowledgments}
This work was supported by JST CREST Grant Number JPMJCR2114
and JSPS KAKENHI Grant Number JP20K20492.
\end{acknowledgments}

\bibliography{main}

\begin{thebibliography}{34}%
\makeatletter
\providecommand \@ifxundefined [1]{%
 \@ifx{#1\undefined}
}%
\providecommand \@ifnum [1]{%
 \ifnum #1\expandafter \@firstoftwo
 \else \expandafter \@secondoftwo
 \fi
}%
\providecommand \@ifx [1]{%
 \ifx #1\expandafter \@firstoftwo
 \else \expandafter \@secondoftwo
 \fi
}%
\providecommand \natexlab [1]{#1}%
\providecommand \enquote  [1]{``#1''}%
\providecommand \bibnamefont  [1]{#1}%
\providecommand \bibfnamefont [1]{#1}%
\providecommand \citenamefont [1]{#1}%
\providecommand \href@noop [0]{\@secondoftwo}%
\providecommand \href [0]{\begingroup \@sanitize@url \@href}%
\providecommand \@href[1]{\@@startlink{#1}\@@href}%
\providecommand \@@href[1]{\endgroup#1\@@endlink}%
\providecommand \@sanitize@url [0]{\catcode `\\12\catcode `\$12\catcode
  `\&12\catcode `\#12\catcode `\^12\catcode `\_12\catcode `\%12\relax}%
\providecommand \@@startlink[1]{}%
\providecommand \@@endlink[0]{}%
\providecommand \url  [0]{\begingroup\@sanitize@url \@url }%
\providecommand \@url [1]{\endgroup\@href {#1}{\urlprefix }}%
\providecommand \urlprefix  [0]{URL }%
\providecommand \Eprint [0]{\href }%
\providecommand \doibase [0]{https://doi.org/}%
\providecommand \selectlanguage [0]{\@gobble}%
\providecommand \bibinfo  [0]{\@secondoftwo}%
\providecommand \bibfield  [0]{\@secondoftwo}%
\providecommand \translation [1]{[#1]}%
\providecommand \BibitemOpen [0]{}%
\providecommand \bibitemStop [0]{}%
\providecommand \bibitemNoStop [0]{.\EOS\space}%
\providecommand \EOS [0]{\spacefactor3000\relax}%
\providecommand \BibitemShut  [1]{\csname bibitem#1\endcsname}%
\let\auto@bib@innerbib\@empty
\bibitem [{\citenamefont {Abry}\ \emph {et~al.}(2003)\citenamefont {Abry},
  \citenamefont {Flandrin}, \citenamefont {Taqqu} \emph
  {et~al.}}]{abry2003self}%
  \BibitemOpen
  \bibfield  {author} {\bibinfo {author} {\bibnamefont {Abry}, \bibfnamefont
  {Patrice}}, \bibinfo {author} {\bibfnamefont {Patrick}\ \bibnamefont
  {Flandrin}}, \bibinfo {author} {\bibfnamefont {Murad~S}\ \bibnamefont
  {Taqqu}},  \emph {et~al.}} (\bibinfo {year} {2003}),\ \bibfield  {title}
  {\enquote {\bibinfo {title} {Self-similarity and long-range dependence
  through the wavelet lens},}\ }\href@noop {} {\bibfield  {journal} {\bibinfo
  {journal} {Theory and applications of long-range dependence}\ }\textbf
  {\bibinfo {volume} {1}},\ \bibinfo {pages} {527--556}}\BibitemShut {NoStop}%
\bibitem [{\citenamefont {Altmann}\ \emph {et~al.}(2012)\citenamefont
  {Altmann}, \citenamefont {Cristadoro},\ and\ \citenamefont
  {Esposti}}]{altmann2012}%
  \BibitemOpen
  \bibfield  {author} {\bibinfo {author} {\bibnamefont {Altmann}, \bibfnamefont
  {Edouard~G}}, \bibinfo {author} {\bibfnamefont {Giampaolo}\ \bibnamefont
  {Cristadoro}}, and\ \bibinfo {author} {\bibfnamefont {Mirko~D.}\ \bibnamefont
  {Esposti}}} (\bibinfo {year} {2012}),\ \bibfield  {title} {\enquote {\bibinfo
  {title} {On the origin of long-range correlations in texts},}\ }\href@noop {}
  {\bibfield  {journal} {\bibinfo  {journal} {Proceedings of the National
  Academy of Sciences}\ }\textbf {\bibinfo {volume} {109}}~(\bibinfo {number}
  {29}),\ \bibinfo {pages} {11582--11587}}\BibitemShut {NoStop}%
\bibitem [{\citenamefont {Altmann}\ \emph {et~al.}(2009)\citenamefont
  {Altmann}, \citenamefont {Pierrehumbert},\ and\ \citenamefont
  {Motter}}]{altmann2009}%
  \BibitemOpen
  \bibfield  {author} {\bibinfo {author} {\bibnamefont {Altmann}, \bibfnamefont
  {Eduardo~G}}, \bibinfo {author} {\bibfnamefont {Janet~B.}\ \bibnamefont
  {Pierrehumbert}}, and\ \bibinfo {author} {\bibfnamefont {Adilson~E.}\
  \bibnamefont {Motter}}} (\bibinfo {year} {2009}),\ \bibfield  {title}
  {\enquote {\bibinfo {title} {Beyond word frequency: Bursts, lulls, and
  scaling in the temporal distributions of words},}\ }\href@noop {} {\bibfield
  {journal} {\bibinfo  {journal} {PLoS One}\ }\textbf {\bibinfo {volume}
  {4}}~(\bibinfo {number} {e7678})}\BibitemShut {NoStop}%
\bibitem [{\citenamefont {Amari}(2012)}]{amari2012differential}%
  \BibitemOpen
  \bibfield  {author} {\bibinfo {author} {\bibnamefont {Amari}, \bibfnamefont
  {Shun-ichi}}} (\bibinfo {year} {2012}),\ \href@noop {} {\emph {\bibinfo
  {title} {Differential-geometrical methods in statistics}}},\ Vol.~\bibinfo
  {volume} {28}\ (\bibinfo  {publisher} {Springer Science \& Business
  Media})\BibitemShut {NoStop}%
\bibitem [{\citenamefont {Ausloos}(2012)}]{ausloos2012measuring}%
  \BibitemOpen
  \bibfield  {author} {\bibinfo {author} {\bibnamefont {Ausloos}, \bibfnamefont
  {Marcel}}} (\bibinfo {year} {2012}),\ \bibfield  {title} {\enquote {\bibinfo
  {title} {Measuring complexity with multifractals in texts. translation
  effects},}\ }\href@noop {} {\bibfield  {journal} {\bibinfo  {journal} {Chaos,
  Solitons \& Fractals}\ }\textbf {\bibinfo {volume} {45}}~(\bibinfo {number}
  {11}),\ \bibinfo {pages} {1349--1357}}\BibitemShut {NoStop}%
\bibitem [{\citenamefont {Balogh}\ and\ \citenamefont
  {Iseli}(2016)}]{balogh2016dimensions}%
  \BibitemOpen
  \bibfield  {author} {\bibinfo {author} {\bibnamefont {Balogh}, \bibfnamefont
  {Zolt{\'a}n}}, and\ \bibinfo {author} {\bibfnamefont {Annina}\ \bibnamefont
  {Iseli}}} (\bibinfo {year} {2016}),\ \bibfield  {title} {\enquote {\bibinfo
  {title} {Dimensions of projections of sets on riemannian surfaces of constant
  curvature},}\ }\href@noop {} {\bibfield  {journal} {\bibinfo  {journal}
  {Proceedings of the American Mathematical Society}\ }\textbf {\bibinfo
  {volume} {144}}~(\bibinfo {number} {7}),\ \bibinfo {pages}
  {2939--2951}}\BibitemShut {NoStop}%
\bibitem [{\citenamefont {Barab{\'a}si}\ and\ \citenamefont
  {Albert}(1999)}]{barabasi1999emergence}%
  \BibitemOpen
  \bibfield  {author} {\bibinfo {author} {\bibnamefont {Barab{\'a}si},
  \bibfnamefont {Albert-L{\'a}szl{\'o}}}, and\ \bibinfo {author} {\bibfnamefont
  {R{\'e}ka}\ \bibnamefont {Albert}}} (\bibinfo {year} {1999}),\ \bibfield
  {title} {\enquote {\bibinfo {title} {Emergence of scaling in random
  networks},}\ }\href@noop {} {\bibfield  {journal} {\bibinfo  {journal}
  {science}\ }\textbf {\bibinfo {volume} {286}}~(\bibinfo {number} {5439}),\
  \bibinfo {pages} {509--512}}\BibitemShut {NoStop}%
\bibitem [{\citenamefont {Blei}\ \emph {et~al.}(2003)\citenamefont {Blei},
  \citenamefont {Ng},\ and\ \citenamefont {Jordan}}]{blei2003latent}%
  \BibitemOpen
  \bibfield  {author} {\bibinfo {author} {\bibnamefont {Blei}, \bibfnamefont
  {David~M}}, \bibinfo {author} {\bibfnamefont {Andrew~Y}\ \bibnamefont {Ng}},
  and\ \bibinfo {author} {\bibfnamefont {Michael~I}\ \bibnamefont {Jordan}}}
  (\bibinfo {year} {2003}),\ \bibfield  {title} {\enquote {\bibinfo {title}
  {Latent dirichlet allocation},}\ }\href@noop {} {\bibfield  {journal}
  {\bibinfo  {journal} {Journal of machine Learning research}\ }\textbf
  {\bibinfo {volume} {3}}~(\bibinfo {number} {Jan}),\ \bibinfo {pages}
  {993--1022}}\BibitemShut {NoStop}%
\bibitem [{\citenamefont {Brown}\ \emph {et~al.}(2020)\citenamefont {Brown},
  \citenamefont {Mann}, \citenamefont {Ryder} \emph
  {et~al.}}]{brown2020language}%
  \BibitemOpen
  \bibfield  {author} {\bibinfo {author} {\bibnamefont {Brown}, \bibfnamefont
  {Tom}}, \bibinfo {author} {\bibfnamefont {Benjamin}\ \bibnamefont {Mann}},
  \bibinfo {author} {\bibfnamefont {Nick}\ \bibnamefont {Ryder}},  \emph
  {et~al.}} (\bibinfo {year} {2020}),\ \bibfield  {title} {\enquote {\bibinfo
  {title} {Language models are few-shot learners},}\ }\href@noop {} {\bibfield
  {journal} {\bibinfo  {journal} {Advances in neural information processing
  systems}\ }\textbf {\bibinfo {volume} {33}},\ \bibinfo {pages}
  {1877--1901}}\BibitemShut {NoStop}%
\bibitem [{\citenamefont {Chomsky}(2014)}]{chomsky2014aspects}%
  \BibitemOpen
  \bibfield  {author} {\bibinfo {author} {\bibnamefont {Chomsky}, \bibfnamefont
  {Noam}}} (\bibinfo {year} {2014}),\ \href@noop {} {\emph {\bibinfo {title}
  {Aspects of the Theory of Syntax}}},\ \bibinfo {number} {11}\ (\bibinfo
  {publisher} {MIT press})\BibitemShut {NoStop}%
\bibitem [{\citenamefont {Copet}\ \emph {et~al.}(2023)\citenamefont {Copet},
  \citenamefont {Kreuk}, \citenamefont {Gat} \emph {et~al.}}]{copet2023simple}%
  \BibitemOpen
  \bibfield  {author} {\bibinfo {author} {\bibnamefont {Copet}, \bibfnamefont
  {Jade}}, \bibinfo {author} {\bibfnamefont {Felix}\ \bibnamefont {Kreuk}},
  \bibinfo {author} {\bibfnamefont {Itai}\ \bibnamefont {Gat}},  \emph
  {et~al.}} (\bibinfo {year} {2023}),\ \bibfield  {title} {\enquote {\bibinfo
  {title} {Simple and controllable music generation},}\ }in\ \href@noop {}
  {\emph {\bibinfo {booktitle} {Thirty-seventh Conference on Neural Information
  Processing Systems}}}\BibitemShut {NoStop}%
\bibitem [{\citenamefont {Doxas}\ \emph {et~al.}(2010)\citenamefont {Doxas},
  \citenamefont {Dennis},\ and\ \citenamefont
  {Oliver}}]{doxas2010dimensionality}%
  \BibitemOpen
  \bibfield  {author} {\bibinfo {author} {\bibnamefont {Doxas}, \bibfnamefont
  {Isidoros}}, \bibinfo {author} {\bibfnamefont {Simon}\ \bibnamefont
  {Dennis}}, and\ \bibinfo {author} {\bibfnamefont {William~L}\ \bibnamefont
  {Oliver}}} (\bibinfo {year} {2010}),\ \bibfield  {title} {\enquote {\bibinfo
  {title} {The dimensionality of discourse},}\ }\href@noop {} {\bibfield
  {journal} {\bibinfo  {journal} {Proceedings of the National Academy of
  Sciences}\ }\textbf {\bibinfo {volume} {107}}~(\bibinfo {number} {11}),\
  \bibinfo {pages} {4866--4871}}\BibitemShut {NoStop}%
\bibitem [{\citenamefont {Falconer}(2004)}]{falconer2004fractal}%
  \BibitemOpen
  \bibfield  {author} {\bibinfo {author} {\bibnamefont {Falconer},
  \bibfnamefont {Kenneth}}} (\bibinfo {year} {2004}),\ \href@noop {} {\emph
  {\bibinfo {title} {Fractal geometry: mathematical foundations and
  applications}}}\ (\bibinfo  {publisher} {John Wiley \& Sons})\BibitemShut
  {NoStop}%
\bibitem [{\citenamefont {Farmer}(1982)}]{farmer1982information}%
  \BibitemOpen
  \bibfield  {author} {\bibinfo {author} {\bibnamefont {Farmer}, \bibfnamefont
  {J~Doyne}}} (\bibinfo {year} {1982}),\ \bibfield  {title} {\enquote {\bibinfo
  {title} {Information dimension and the probabilistic structure of chaos},}\
  }\href@noop {} {\bibfield  {journal} {\bibinfo  {journal} {Zeitschrift
  f{\"u}r Naturforschung A}\ }\textbf {\bibinfo {volume} {37}}~(\bibinfo
  {number} {11}),\ \bibinfo {pages} {1304--1326}}\BibitemShut {NoStop}%
\bibitem [{\citenamefont {Grassberger}\ and\ \citenamefont
  {Procaccia}(1983)}]{grassberger1983characterization}%
  \BibitemOpen
  \bibfield  {author} {\bibinfo {author} {\bibnamefont {Grassberger},
  \bibfnamefont {Peter}}, and\ \bibinfo {author} {\bibfnamefont {Itamar}\
  \bibnamefont {Procaccia}}} (\bibinfo {year} {1983}),\ \bibfield  {title}
  {\enquote {\bibinfo {title} {Characterization of strange attractors},}\
  }\href@noop {} {\bibfield  {journal} {\bibinfo  {journal} {Physical review
  letters}\ }\textbf {\bibinfo {volume} {50}}~(\bibinfo {number} {5}),\
  \bibinfo {pages} {346}}\BibitemShut {NoStop}%
\bibitem [{\citenamefont {Kershaw}\ and\ \citenamefont
  {Koeling}(2020)}]{Kershaw2020ElsevierOC}%
  \BibitemOpen
  \bibfield  {author} {\bibinfo {author} {\bibnamefont {Kershaw}, \bibfnamefont
  {Daniel~James}}, and\ \bibinfo {author} {\bibfnamefont {R.}~\bibnamefont
  {Koeling}}} (\bibinfo {year} {2020}),\ \bibfield  {title} {\enquote {\bibinfo
  {title} {Elsevier oa cc-by corpus},}\ }\href
  {https://elsevier.digitalcommonsdata.com/datasets/zm33cdndxs} {\bibfield
  {journal} {\bibinfo  {journal} {ArXiv}\ }\textbf {\bibinfo {volume}
  {abs/2008.00774}}}\BibitemShut {NoStop}%
\bibitem [{\citenamefont {Kobayashi}\ and\ \citenamefont
  {Tanaka-Ishii}(2018)}]{acl18}%
  \BibitemOpen
  \bibfield  {author} {\bibinfo {author} {\bibnamefont {Kobayashi},
  \bibfnamefont {Tatsuru}}, and\ \bibinfo {author} {\bibfnamefont {Kumiko}\
  \bibnamefont {Tanaka-Ishii}}} (\bibinfo {year} {2018}),\ \bibfield  {title}
  {\enquote {\bibinfo {title} {Taylor's law for human linguistic sequences},}\
  }\href@noop {} {\bibinfo  {journal} {Proceedings of the 56th Annual Meeting
  of the Association for Computational Lingusitics}\ ,\ \bibinfo {pages}
  {1138--1148}}\BibitemShut {NoStop}%
\bibitem [{\citenamefont {Lacasa}\ and\ \citenamefont
  {G{\'o}mez-Gardenes}(2013)}]{lacasa2013correlation}%
  \BibitemOpen
\bibfield  {journal} {  }\bibfield  {author} {\bibinfo {author} {\bibnamefont
  {Lacasa}, \bibfnamefont {Lucas}}, and\ \bibinfo {author} {\bibfnamefont
  {Jes{\'u}s}\ \bibnamefont {G{\'o}mez-Gardenes}}} (\bibinfo {year} {2013}),\
  \bibfield  {title} {\enquote {\bibinfo {title} {Correlation dimension of
  complex networks},}\ }\href@noop {} {\bibfield  {journal} {\bibinfo
  {journal} {Physical review letters}\ }\textbf {\bibinfo {volume}
  {110}}~(\bibinfo {number} {16}),\ \bibinfo {pages} {168703}}\BibitemShut
  {NoStop}%
\bibitem [{\citenamefont {Li}(1989)}]{li1989mutual}%
  \BibitemOpen
  \bibfield  {author} {\bibinfo {author} {\bibnamefont {Li}, \bibfnamefont
  {Wentian}}} (\bibinfo {year} {1989}),\ \bibfield  {title} {\enquote {\bibinfo
  {title} {Mutual information functions of natural language texts},}\ \
  }(\bibinfo {organization} {Citeseer})\BibitemShut {NoStop}%
\bibitem [{\citenamefont {Marstrand}(1954)}]{marstrand1954some}%
  \BibitemOpen
  \bibfield  {author} {\bibinfo {author} {\bibnamefont {Marstrand},
  \bibfnamefont {John~M}}} (\bibinfo {year} {1954}),\ \bibfield  {title}
  {\enquote {\bibinfo {title} {Some fundamental geometrical properties of plane
  sets of fractional dimensions},}\ }\href@noop {} {\bibfield  {journal}
  {\bibinfo  {journal} {Proceedings of the London Mathematical Society}\
  }\textbf {\bibinfo {volume} {3}}~(\bibinfo {number} {1}),\ \bibinfo {pages}
  {257--302}}\BibitemShut {NoStop}%
\bibitem [{\citenamefont {OpenAI}(2023)}]{openai2023gpt4}%
  \BibitemOpen
  \bibfield  {author} {\bibinfo {author} {\bibnamefont {OpenAI},}} (\bibinfo
  {year} {2023}),\ \href@noop {} {\enquote {\bibinfo {title} {Gpt-4 technical
  report},}\ }\Eprint {https://arxiv.org/abs/2303.08774} {arXiv:2303.08774
  [cs.CL]} \BibitemShut {NoStop}%
\bibitem [{\citenamefont {Osborne}\ and\ \citenamefont
  {Provenzale}(1989)}]{osborne1989finite}%
  \BibitemOpen
  \bibfield  {author} {\bibinfo {author} {\bibnamefont {Osborne}, \bibfnamefont
  {A~Ro}}, and\ \bibinfo {author} {\bibfnamefont {A}~\bibnamefont
  {Provenzale}}} (\bibinfo {year} {1989}),\ \bibfield  {title} {\enquote
  {\bibinfo {title} {Finite correlation dimension for stochastic systems with
  power-law spectra},}\ }\href@noop {} {\bibfield  {journal} {\bibinfo
  {journal} {Physica D: Nonlinear Phenomena}\ }\textbf {\bibinfo {volume}
  {35}}~(\bibinfo {number} {3}),\ \bibinfo {pages} {357--381}}\BibitemShut
  {NoStop}%
\bibitem [{\citenamefont {Peitgen}\ \emph {et~al.}(1992)\citenamefont
  {Peitgen}, \citenamefont {J{\"u}rgens}, \citenamefont {Saupe},\ and\
  \citenamefont {Feigenbaum}}]{peitgen1992chaos}%
  \BibitemOpen
  \bibfield  {author} {\bibinfo {author} {\bibnamefont {Peitgen}, \bibfnamefont
  {Heinz-Otto}}, \bibinfo {author} {\bibfnamefont {Hartmut}\ \bibnamefont
  {J{\"u}rgens}}, \bibinfo {author} {\bibfnamefont {Dietmar}\ \bibnamefont
  {Saupe}}, and\ \bibinfo {author} {\bibfnamefont {Mitchell~J}\ \bibnamefont
  {Feigenbaum}}} (\bibinfo {year} {1992}),\ \href@noop {} {\emph {\bibinfo
  {title} {Chaos and fractals: new frontiers of science}}},\ Vol.~\bibinfo
  {volume} {7}\ (\bibinfo  {publisher} {Springer})\BibitemShut {NoStop}%
\bibitem [{\citenamefont {Pesin}(1993)}]{pesin1993rigorous}%
  \BibitemOpen
  \bibfield  {author} {\bibinfo {author} {\bibnamefont {Pesin}, \bibfnamefont
  {Ya~B}}} (\bibinfo {year} {1993}),\ \bibfield  {title} {\enquote {\bibinfo
  {title} {On rigorous mathematical definitions of correlation dimension and
  generalized spectrum for dimensions},}\ }\href@noop {} {\bibfield  {journal}
  {\bibinfo  {journal} {Journal of statistical physics}\ }\textbf {\bibinfo
  {volume} {71}},\ \bibinfo {pages} {529--547}}\BibitemShut {NoStop}%
\bibitem [{\citenamefont {Radford}\ \emph {et~al.}(2019)\citenamefont
  {Radford}, \citenamefont {Wu}, \citenamefont {Child} \emph
  {et~al.}}]{radford2019language}%
  \BibitemOpen
  \bibfield  {author} {\bibinfo {author} {\bibnamefont {Radford}, \bibfnamefont
  {Alec}}, \bibinfo {author} {\bibfnamefont {Jeffrey}\ \bibnamefont {Wu}},
  \bibinfo {author} {\bibfnamefont {Rewon}\ \bibnamefont {Child}},  \emph
  {et~al.}} (\bibinfo {year} {2019}),\ \bibfield  {title} {\enquote {\bibinfo
  {title} {Language models are unsupervised multitask learners},}\ }\href@noop
  {} {\bibfield  {journal} {\bibinfo  {journal} {OpenAI blog}\ }\textbf
  {\bibinfo {volume} {1}}~(\bibinfo {number} {8}),\ \bibinfo {pages}
  {9}}\BibitemShut {NoStop}%
\bibitem [{\citenamefont {Rak}\ and\ \citenamefont
  {Rak}(2020)}]{rak2020fractional}%
  \BibitemOpen
  \bibfield  {author} {\bibinfo {author} {\bibnamefont {Rak}, \bibfnamefont
  {Rafa{\l}}}, and\ \bibinfo {author} {\bibfnamefont {Ewa}\ \bibnamefont
  {Rak}}} (\bibinfo {year} {2020}),\ \bibfield  {title} {\enquote {\bibinfo
  {title} {The fractional preferential attachment scale-free network model},}\
  }\href@noop {} {\bibfield  {journal} {\bibinfo  {journal} {Entropy}\ }\textbf
  {\bibinfo {volume} {22}}~(\bibinfo {number} {5}),\ \bibinfo {pages}
  {509}}\BibitemShut {NoStop}%
\bibitem [{\citenamefont {Rao}(1992)}]{rao1992information}%
  \BibitemOpen
  \bibfield  {author} {\bibinfo {author} {\bibnamefont {Rao}, \bibfnamefont
  {C~Radhakrishna}}} (\bibinfo {year} {1992}),\ \bibfield  {title} {\enquote
  {\bibinfo {title} {Information and the accuracy attainable in the estimation
  of statistical parameters},}\ }in\ \href@noop {} {\emph {\bibinfo {booktitle}
  {Breakthroughs in Statistics: Foundations and basic theory}}}\ (\bibinfo
  {publisher} {Springer})\ pp.\ \bibinfo {pages} {235--247}\BibitemShut
  {NoStop}%
\bibitem [{\citenamefont {Russell}\ \emph {et~al.}(1980)\citenamefont
  {Russell}, \citenamefont {Hanson},\ and\ \citenamefont
  {Ott}}]{russell1980dimension}%
  \BibitemOpen
  \bibfield  {author} {\bibinfo {author} {\bibnamefont {Russell}, \bibfnamefont
  {David~A}}, \bibinfo {author} {\bibfnamefont {James~D}\ \bibnamefont
  {Hanson}}, and\ \bibinfo {author} {\bibfnamefont {Edward}\ \bibnamefont
  {Ott}}} (\bibinfo {year} {1980}),\ \bibfield  {title} {\enquote {\bibinfo
  {title} {Dimension of strange attractors},}\ }\href@noop {} {\bibfield
  {journal} {\bibinfo  {journal} {Physical Review Letters}\ }\textbf {\bibinfo
  {volume} {45}}~(\bibinfo {number} {14}),\ \bibinfo {pages}
  {1175}}\BibitemShut {NoStop}%
\bibitem [{\citenamefont {Simon}(1955)}]{simon1955class}%
  \BibitemOpen
  \bibfield  {author} {\bibinfo {author} {\bibnamefont {Simon}, \bibfnamefont
  {Herbert~A}}} (\bibinfo {year} {1955}),\ \bibfield  {title} {\enquote
  {\bibinfo {title} {On a class of skew distribution functions},}\ }\href@noop
  {} {\bibfield  {journal} {\bibinfo  {journal} {Biometrika}\ }\textbf
  {\bibinfo {volume} {42}}~(\bibinfo {number} {3/4}),\ \bibinfo {pages}
  {425--440}}\BibitemShut {NoStop}%
\bibitem [{\citenamefont {Tanaka-Ishii}\ and\ \citenamefont
  {Bunde}(2016)}]{plosone16}%
  \BibitemOpen
  \bibfield  {author} {\bibinfo {author} {\bibnamefont {Tanaka-Ishii},
  \bibfnamefont {Kumiko}}, and\ \bibinfo {author} {\bibfnamefont {Armin}\
  \bibnamefont {Bunde}}} (\bibinfo {year} {2016}),\ \bibfield  {title}
  {\enquote {\bibinfo {title} {Long-range memory in literary texts: On the
  universal clustering of the rare words},}\ }\href
  {https://doi.org/10.1371/journal.pone.0164658} {\bibfield  {journal}
  {\bibinfo  {journal} {PLoS One}\ }\textbf {\bibinfo {volume} {11}}~(\bibinfo
  {number} {11}),\ \bibinfo {pages} {e0164658}}\BibitemShut {NoStop}%
\bibitem [{\citenamefont {Tanaka-Ishii}\ and\ \citenamefont
  {Kobayashi}(2018)}]{jpc18}%
  \BibitemOpen
  \bibfield  {author} {\bibinfo {author} {\bibnamefont {Tanaka-Ishii},
  \bibfnamefont {Kumiko}}, and\ \bibinfo {author} {\bibfnamefont {Tatsuru}\
  \bibnamefont {Kobayashi}}} (\bibinfo {year} {2018}),\ \bibfield  {title}
  {\enquote {\bibinfo {title} {Taylor's law for linguistic sequences and random
  walk models},}\ }\href {https://doi.org/10.1088/2399-6528/aaefb2} {\bibfield
  {journal} {\bibinfo  {journal} {Journal of Physics Communications}\ }\textbf
  {\bibinfo {volume} {2}}~(\bibinfo {number} {11}),\ \bibinfo {pages}
  {089401}}\BibitemShut {NoStop}%
\bibitem [{\citenamefont {Touvron}\ \emph {et~al.}(2023)\citenamefont
  {Touvron}, \citenamefont {Martin}, \citenamefont {Stone} \emph
  {et~al.}}]{touvron2023llama}%
  \BibitemOpen
  \bibfield  {author} {\bibinfo {author} {\bibnamefont {Touvron}, \bibfnamefont
  {Hugo}}, \bibinfo {author} {\bibfnamefont {Louis}\ \bibnamefont {Martin}},
  \bibinfo {author} {\bibfnamefont {Kevin}\ \bibnamefont {Stone}},  \emph
  {et~al.}} (\bibinfo {year} {2023}),\ \bibfield  {title} {\enquote {\bibinfo
  {title} {Llama 2: Open foundation and fine-tuned chat models},}\ }\href@noop
  {} {\bibinfo  {journal} {arXiv preprint arXiv:2307.09288}\ }\BibitemShut
  {NoStop}%
\bibitem [{\citenamefont {Tzanetakis}\ and\ \citenamefont
  {Cook}(2002)}]{tzanetakis2002musical}%
  \BibitemOpen
\bibfield  {journal} {  }\bibfield  {author} {\bibinfo {author} {\bibnamefont
  {Tzanetakis}, \bibfnamefont {George}}, and\ \bibinfo {author} {\bibfnamefont
  {Perry}\ \bibnamefont {Cook}}} (\bibinfo {year} {2002}),\ \bibfield  {title}
  {\enquote {\bibinfo {title} {Musical genre classification of audio
  signals},}\ }\href@noop {} {\bibfield  {journal} {\bibinfo  {journal} {IEEE
  Transactions on speech and audio processing}\ }\textbf {\bibinfo {volume}
  {10}}~(\bibinfo {number} {5}),\ \bibinfo {pages} {293--302}}\BibitemShut
  {NoStop}%
\bibitem [{\citenamefont {Yi}(2024)}]{yi}%
  \BibitemOpen
  \bibfield  {author} {\bibinfo {author} {\bibnamefont {Yi},}} (\bibinfo {year}
  {2024}),\ \href@noop {} {\enquote {\bibinfo {title} {The yi model},}\
  }\bibinfo {note} {\url{https://huggingface.co/01-ai/Yi-34B}, visited in
  January 2024}\BibitemShut {NoStop}%
\end{thebibliography}%

\newpage
\balancecolsandclearpage
\appendix
\onecolumngrid

\section{Properties of the Mapping $\phi:x_t\mapsto p_t$}
\subsection{Formulation}
\label{sec:phi}
In the main text, we consider two sequences of probability distributions, i.e., $\{x_t\}$ and $\{p_t\}$, which are related by the mapping $\phi:x_t\mapsto p_t$ in \ref{eq:phi}. We explain the formulation of $\phi$ here.

Recall that $x_t$ denotes the language dynamical system's state at timestep $t$ and is defined as a distribution over the set $\seqs$ of all sequences of words, as in \ref{eq:xt}. Then, $p_t$ is defined over the vocabulary $\vocab$ and characterizes the probability of a word $w$ to occur as the next word following a given context. The probability $p_t(\word)$ for any $\word\in\vocab$ is defined as the probability that an arbitrary closed text has $\word$ as its first word, thus giving the following formulation of $\phi$:
\begin{equation}
  p_t(\word) =
  \phi(x_t)(\word) :=
  \sum_{\substack{\seqbm_{\geq t}\in\seqs \\ \seq_t=\word }}
  x_t(\seqbm_{\geq t})
  ~~~~~~\forall \word\in\vocab.
  \label{eq:phi2}
\end{equation}

$x_t$ and $p_t$ are defined over different sets but are essentially consistent with respect to the same probability measure $\mu$ on a probability space $(\seqs, \mathcal{F}, \mu)$, where $\mathcal{F}=\{\Lambda : \Lambda\subset \seqs\}$ denotes the power set of $\seqs$. Here, $\mu: \mathcal{F}\to[0,1]$ is defined as follows:
\begin{equation}
  \mu(\Lambda) = \sum_{\seqbm_{\geq t}\in \Lambda} x_t(\seqbm_{\geq t}),
  ~~~~~~\forall \Lambda\subset\seqs.
\end{equation}
Hence, $x_t$ are $p_t$ are both consistent with respect to $\mu$:
\begin{alignat}{2}
  &x_t(\seqbm_{\geq t}) = \mu(\{\seqbm_{\geq t}\}) ~~~~~~&& \forall\seqbm_{\geq t}\in\seqs, \\
  &p_t(\word) = \mu\left(
    \coprod_{\substack{\seqbm_{\geq t}\in\seqs \\ \seq_t=\word}} \{\seqbm_{\geq t}\}
    \right) ~~~~~~&& \forall \word\in\vocab.
\end{alignat}

\subsection{Linearity}
\label{sec:linearity}
The mapping $\phi$ is linear with respect to the mixture of probability distributions within $\{x_t\}$. That is, for any two distributions $x_t, x_s$, and any mixture weight $\alpha\in[0,1]$, $\phi$ satisfies the following:
\begin{equation}
  \phi\bigl(\alpha x_t + (1-\alpha) x_s\bigr) = \alpha \phi(x_t) + (1-\alpha) \phi(x_s).
\end{equation}

This equality can be obtained directly from the definition of $\phi$ in Formula (\ref{eq:phi2}). The left side of Formula (\ref{eq:phi2}) is calculated as follows for any $\word\in\vocab$:
\begin{align}
  \phi\bigl(\alpha x_t + (1-\alpha) x_s \bigr)(\word)
  &= \sum_{\substack{\seqbm_{\geq t}\in\seqs \\ \seq_t=\word}}
  \bigl(\alpha x_t + (1-\alpha) x_s\bigr)(\seqbm_{\geq t}) \\
  &= \sum_{\substack{\seqbm_{\geq t}\in\seqs \\ \seq_t=\word}}
  \Bigl(
  \alpha x_t(\seqbm_{\geq t}) + (1-\alpha) x_s(\seqbm_{\geq t})
  \Bigr) \\
  &= \alpha \sum_{\substack{\seqbm_{\geq t}\in\seqs \\ \seq_t=\word}} x_t(\seqbm_{\geq t})
  + (1-\alpha) \sum_{\substack{\seqbm_{\geq s}\in\seqs \\ \seq_s=\word}} x_s(\seqbm_{\geq s}) \\
  &= \alpha p_t(\word) + (1-\alpha) p_s(\word).
\end{align}

\subsection{Distance Distortion Rate}
\label{sec:distortion}
In this work, we are especially interested in how the mapping $\phi: x_t\mapsto p_t$ would distort the Fisher-Rao distance between any two states $x_t$ and $x_s$. The distortion is measured by the following rate:
\begin{equation}
  r(x_t, x_s) \equiv d_\text{FR}(x_t,x_s)/d_\text{FR}(p_t,p_s),
  \label{eq:distortion}
\end{equation}
where $p_t=\phi(x_t)$ and $p_s=\phi(x_s)$ as defined in Formula (\ref{eq:phi2}).

In this section, we show that $r(x_t, x_s)\geq 1$ in general, i.e., it has a lower bound of 1.

\begin{lemma}
(Lower Bound of the Distance Distortion Rate) The distortion rate $r(x_t,x_s)$ is no smaller than 1 for any $x_t$ and $x_s$.
  \label{thm:lowerbound}
\end{lemma}

\begin{proof}
For any $\seqbm_{\geq t}\in\seqs$, $x_t$ (and $x_s$) can be decomposed as follows:
  \begin{align}
    x_t(\seqbm_{\geq t}) &\equiv \prob(\seqbm_{\geq t} \mid \seqbm_{<t}) \\
    &= \prob(\seq_t \mid \seqbm_{<t}) \cdot \prob(\seqbm_{\geq t+1} \mid \seqbm_{<t}, \seq_{t}) \\
    &= p_t(\seq_t) \cdot \prob(\seqbm_{\geq t+1} \mid \seqbm_{<t}, \seq_{t}).
  \end{align}
For simplicity of notation, let $x_{t+1}(\seqbm_{\geq t+1} \mid \seq_t) \equiv \prob(\seqbm_{\geq t+1} \mid \seqbm_{<t}, \seq_{t})$.

Hence, the distance $d_\text{FR}(x_t, x_s)$ can be decomposed as follows:
  \begin{align}
    d_\text{FR}(x_t, x_s)
    &= 2\arccos\sum_{\bm{b}\in\seqs}
    \sqrt{ x_t(\bm{b}) \cdot x_s(\bm{b}) } \\
    &= 2\arccos\sum_{\bm{b}\in\seqs} \sqrt{
      p_t(b_1) x_{t+1}(\bm{b}_{\geq 2}\mid b_1) \cdot
      p_s(b_1) x_{s+1}(\bm{b}_{\geq 2}\mid b_1)
    } \\
    &= 2\arccos\sum_{b_1\in\vocab} \sqrt{p_t(b_1) p_s(b_1)}
    \underbrace{
    \sum_{\bm{b}_{\geq 2}\in\seqs} \sqrt{
      x_{t+1}(\bm{b}_{\geq 2}\mid b_1)
      x_{s+1}(\bm{b}_{\geq 2}\mid b_1)
    } 
    }_{ \cos \Bigl[ \frac{1}{2} d_\text{FR}\bigl(x_{t+1}(\cdot\mid b_1), x_{s+1}(\cdot\mid b_2)\bigr) \Bigr] ~~~~ \leq 1} \\
    &\geq 2\arccos\sum_{b_1\in\vocab} \sqrt{p_t(b_1) p_s(b_1)} \\
    &= d_\text{FR}(p_t,p_s),
  \end{align}
where $\bm{b}\in\seqs$ denotes any closed text, and $b_1$ and $\bm{b}_{\geq 2}=[b_2,b_3\cdots]$ represent the first word and the remainder of the text, respectively. This implies $r(x_t,x_s) \geq 1$.
\end{proof}

\subsection{Dimension Preservation for Markov Processes}
\label{sec:markov}
In this section, we analyze whether $\phi$ preserves the correlation dimension of a language system $\{x_t\}$ when the system follows certain Markov conditions. In other words, we examine whether $\nu=\hat{\nu}$ holds, where $\nu$ and $\hat{\nu}$ are the respective correlation dimensions of $\{x_t\}$ and $\{p_t\}=\{\phi(x_t)\}$.

We consider two kinds of Markov conditions, and we show that under either kind of condition, the distortion rate $r(x_t,x_s)$ defined in Formula (\ref{eq:distortion}) is bounded above. That is, there exists a constant $C$ such that
\begin{equation}
  r(x_t, x_s) < C
  \label{eq:upperbound}
\end{equation}
for any $x_t$ and $x_s$. Because $r(x_t, x_s) \geq 1$ holds in general (see Lemma \ref{thm:lowerbound} in Section \ref{sec:distortion}), the boundedness of $r(x_t, x_s)$ in Formula (\ref{eq:upperbound}) implies the bi-Lipschitz characteristic of $\phi$ and thus the equality $\nu=\hat{\nu}$. 

The first kind of Markov condition specifies the case when $\seqbm_{\geq t}$ and $\seqbm_{\geq s}$ are generated by the same Markov process with different initial states. For the second kind, $\seqbm_{\geq t}$ and $\seqbm_{\geq s}$ follow two different Markov processes, and we consider the case when their initial states get infinitesimally close to each other. The two kinds of conditions are examined in Sections \ref{sec:markov1} and \ref{sec:markov2}, respectively.

\subsubsection{When \seqbmtitle$_{\geq t}$ and  \seqbmtitle$_{\geq s}$ Follow the Same Markov Process}
\label{sec:markov1}
A Markov process can be represented by its transition matrix $A$. A text $\seqbm_{\geq t}$ (or $\seqbm_{\geq s}$) is said to follow a Markov process if $\forall \tau\geq t$ (or $\forall \tau\geq s$),
\begin{align}
  p_\tau &\equiv \prob(\seq_\tau\mid\seqbm_{<\tau})
  =\prob(\seq_\tau\mid \seq_{\tau-1}) =: A_{\seq_{\tau-1}, \seq_\tau},
\end{align}
where $A_{\seq_{\tau-1}, \seq_\tau}$ represents the transition probability from word $\seq_{\tau-1}$ to $\seq_{\tau}$, i.e., the probability that $\seq_{\tau}$ occurs immediately after $\seq_{\tau-1}$. 

\begin{theorem}
Consider two Markov processes that are defined over a vocabulary $\vocab$ and have the same transition matrix $A$ but different initial states $p_t$ and $p_s$. Then, the distance distortion rate $r(x_t, x_s)=1$.
\label{thm:markov1}
\end{theorem}

\begin{proof}
  \begin{align}
    \cos \frac{d_\text{FR}(x_t, x_s)}{2} 
    &= \sum_{\bm{b}\in\seqs}
      \sqrt{x_t(\bm{b}) x_s(\bm{b})} \\
    &= \sum_{b_1\in\vocab} \sqrt{p_t(b_1) p_s(b_1)}
    \sum_{\bm{b}_{\geq 2}\in\seqs}
    \sqrt{x_{t+1}(\bm{b}_{\geq 2}\mid b_1) ~
          x_{s+1}(\bm{b}_{\geq 2}\mid b_1)},
    \label{eq:markov1-1}
  \end{align}
where $x_{t+1}(\bm{b}_{\geq 2}\mid b_1) \equiv\prob(\seqbm_{\geq t+1}=\bm{b}_{\geq 2} \mid \seq_{t}=b_1, \seq_{<t}=\seqbm_{<t})$, and $x_{s+1}(\bm{b}_{\geq 2}\mid b_1)$ is defined similarly. Owing to the Markov property of $\seqbm$, $x_{t+1}(\bm{b}_{\geq 2}\mid b_1)$ is decomposed as follows:
  \begin{equation}
    x_{t+1}(\bm{b}_{\geq 2}\mid b_1)
    = \prod_{\tau=2}^{|\bm{b}|} A_{b_\tau, b_{\tau+1}}
    = x_{s+1}(\bm{b}_{\geq 2}\mid b_1).
  \end{equation}
  
Hence, the second term of Formula (\ref{eq:markov1-1}) simplifies to
  \begin{alignat}{2}
    \sum_{\bm{b}_{\geq 2}\in\seqs}
    \sqrt{x_{t+1}(\bm{b}_{\geq 2}\mid b_1) ~
          x_{s+1}(\bm{b}_{\geq 2}\mid b_1)}
    = \sum_{\bm{b}_{\geq 2}\in\seqs} x_{t+1}(\bm{b}_{\geq 2} \mid b_1) = 1.
  \end{alignat}
  
Therefore,
  \begin{align}
    \cos \frac{d_\text{FR}(x_t, x_s)}{2} 
    = \sum_{b_1\in\vocab}
    \sqrt{p_t(b_1) p_s(b_1)} 
    = \cos \frac{d_\text{FR}(p_t, p_s)}{2},
  \end{align}
which implies $d_\text{FR}(x_t, x_s)=d_\text{FR}(p_t, p_s)$ and thus $r(x_t, x_s)=1$ for any $x_t$ and $x_s$.
\end{proof}

\subsubsection{When \seqbmtitle$_{\geq t}$ and \seqbmtitle$_{\geq s}$ Follow  Different Markov Processes}
\label{sec:markov2}
For the second kind of Markov condition, we assume the two word sequences $\seqbm_{\geq t}$ and $\seqbm_{\geq s}$ to be generated by two separate Markov processes. The transition matrices are denoted as $A$ and $B$, respectively.

The difference between $A$ and $B$ is quantified by a row-wise metric $\Delta_\word$ ($\word\in\vocab$) that is defined as the Fisher-Rao distance between the transition probabilities from state $\word$ to all states:
\begin{equation}
  \Delta_\word = 2 \arccos \sum_{j\in\vocab} \sqrt{A_{\word,j} B_{\word,j}}.
\end{equation}
The average difference across the columns $j$ is reflected as the distance between the next-word probability distributions. For simplicity, we restrict our discussion to Markov processes that satisfy the following:
\begin{equation}
  \Delta_\word \leq d_\text{FR}(p_t, p_s) ~~~~ \word\in\vocab.
  \label{eq:markov2-norm}
\end{equation}
This is analogous to bounding the matrix norm $\lVert A-B\rVert$ by $d_\text{FR}(p_t, p_s)$.

Here, the calculation of $d_\text{FR}(x_t,x_s)$ is more difficult than in the case above, and we must consider texts of different lengths. In other words, the ``end'' of a text must be clearly defined. $\seqs$ was abstractly defined as ``the set of all texts,'' but this definition must be made rigorous here to incorporate the text lengths. 

Hence, we restrict $\seqs$ to the set of all {\em closed} texts, i.e.,
those end with a special closing token denoted as \eos. Note that the
probabilities of these closed texts also determine the probability of
any unclosed text (i.e., one without \eos) through aggregation of all
closed texts that have the unclosed text as their prefix. Each word in
the vocabulary besides \eos can be seen as an unclosed text of length
one. 

For the two Markov processes, \eos can be understood as an ``absorbing state,'' because the transition probability from \eos to any other state is zero: a closed text will not return to being unclosed. In other words, $A_{\eos,\eos}=1$, and $A_{\eos,\word}=0$ ($\forall\word\in\vocab\backslash\eos$). We consider the simplest case in which the transition probabilities to \eos are equal, with a value denoted as $\rho$:
\begin{equation}
  A_{\word,\eos} = B_{\word,\eos} = \rho ~~~~~~\forall \word\in\vocab,
  \label{eq:markov2-eos}
\end{equation}
where $0<\rho<1$ is a constant.

\begin{theorem}
For a pair of Markov processes $A$ and $B$ defined over $\vocab$ with an absorbing state \eos, if the two conditions in Formulas (\ref{eq:markov2-norm}) and (\ref{eq:markov2-eos}) are met, then
  \begin{equation}
    \lim_{d_\text{FR}(p_t, p_s)\to 0} r(x_t, x_s)
    \leq \rho^{-1/2}.
  \end{equation}
  \label{thm:markov2}
\end{theorem}

\begin{proof}
Recall that the Fisher-Rao distances between $x_t$ and $x_s$ and between $p_t$ and $p_s$ are defined as follows:
  \begin{align}
    d_\text{FR}(x_t, x_s) &= 2\arccos \sum_{\seq\in\seqs}
    \sqrt{x_t(\seqbm) x_s(\seqbm)}, \\
    d_\text{FR}(p_t, p_s) &= 2\arccos \sum_{\word\in\vocab}
    \sqrt{p_t(\word) p_s(\word)}.
  \end{align}

Rearranging the sequences within $\seqs$ by considering different sequence lengths $n$ from 1 to infinity, we have the following:
  \begin{align}
    d_\text{FR}(x_t, x_s)
    = 2\arccos \sum_{\substack{\text{closed}\,\seqbm}}
    \sqrt{x_t(\seqbm) x_s(\seqbm)}
    = 2\arccos \sum_{n=1}^\infty H_n,
    \label{eq:cos-delta}
  \end{align}
where
  \begin{equation}
    H_n = 
    \sum_{\substack{\text{unclosed}\,\seqbm \\ |\seqbm|=n-1}}
    \sqrt{x_t([\seqbm, \eos]) \cdot x_s([\seqbm, \eos])}.
  \end{equation}
Here, $|\seqbm|$ notes the length of the sequence $\seqbm$;
$[\seqbm,\eos]$ represents a closed text formed by concatenating the
unclosed text $\seqbm$ with $\eos$. In particular, $H_1 = \rho$.

By considering the last word $\partial\seq$ of any sequence $\seqbm$, we can see that
  \begin{align}
    H_{n+1} &=
    \sum_{\substack{\text{unclosed}\,\seqbm \\ |\seqbm|=n-1}}
      \sqrt{
        \frac{x_t([\seqbm, \eos])}{A_{\partial\seq, \eos}}\cdot
        \sum_{\word\in\vocab\backslash\eos} A_{\partial\seq, \word} A_{\word,\eos}
      } \\
      &\qquad\qquad \cdot
      \sqrt{
        \frac{x_s([\seqbm, \eos])}{B_{\partial\seq, \eos}}\cdot
        \sum_{\word\in\vocab\backslash\eos} B_{\partial\seq, \word} B_{\word,\eos}
      }
      \\
    &= 
    \sum_{\substack{\text{unclosed}\,\seq \\ |\seq|=n-1 \\ \seq_1=i}}
    \left\{
      \sqrt{
        x_t([\seqbm, \eos])
        \cdot
        x_s([\seqbm, \eos])
      }
      \cdot
      \sum_{\word\in\vocab\backslash\eos} \sqrt{A_{\partial\seq,\word} B_{\partial\seq,\word}}
    \right\} \\
    &= 
    \sum_{\substack{\text{unclosed}\,\seq \\ |\seq|=n-1 \\ \seq_1=i}}
    \left\{
      \sqrt{
        x_t([\seqbm, \eos])
        \cdot
        x_t([\seqbm, \eos])
      }
      \cdot
      \left(\cos \frac{\Delta_{\partial\seq}}{2} - \rho\right)
    \right\} \\
    &\geq
    \sum_{\substack{\text{unclosed}\,\seq \\ |\seq|=n-1 \\ \seq_1=i}}
    \left\{
      \sqrt{
        x_t([\seqbm, \eos])
        \cdot
        x_t([\seqbm, \eos])
      }
      \cdot
      \left(\cos \frac{d_\text{FR}(p_t, p_s)}{2} - \rho\right)
    \right\} \\
    &=
      \left(\cos \frac{d_\text{FR}(p_t, p_s)}{2} - \rho\right)
      H_n,
      \label{eq:sn+1}
  \end{align}
where the inequality is due to the condition in Formula (\ref{eq:markov2-norm}).

Next, by combining Formulas (\ref{eq:sn+1}) and (\ref{eq:cos-delta}) and applying the formula for the sum of a geometric series, we have
  \begin{equation}
    d_\text{FR}(x_t,x_s) \leq 2 \arccos \frac{\rho}{1+\rho - \cos\frac{d_\text{FR}(p_t,p_s)}{2}} =: F(d_\text{FR}(p_t,p_s)).
  \end{equation}
When $d_\text{FR}(p_t,p_s) \to 0$, it follows that $F(d_\text{FR}(p_t,p_s)) \to 0$ as well. Furthermore,
  \begin{align}
    \lim_{d_\text{FR}(p_t,p_s) \to 0} r(x_t,x_s)
    \leq\lim_{d_\text{FR}(p_t,p_s) \to 0} \frac{F(d_\text{FR}(p_t,p_s))}{d_\text{FR}(p_t,p_s)}
    = \rho^{-1/2},
  \end{align}
which implies Theorem \ref{thm:markov2}.
\end{proof}

Furthermore, as all i.i.d. processes meet the two conditions in Formulas (\ref{eq:markov2-norm}) and (\ref{eq:markov2-eos}), we immediately obtain the following corollary for i.i.d. processes.
\begin{corollary}
Consider two i.i.d. processes that are defined over a vocabulary $\vocab$ and represented by probability vectors $\bm{u}$ and $\bm{v}$. The first entries of $\bm{u}$ and $\bm{v}$, denoted respectively as $u_1$ and $v_1$, specify the occurrence probability of \eos at any timestep. Then, it follows that
  \begin{equation}
    \lim_{d_\text{FR}(p_t, p_s)\to 0} r(x_t, x_s)
    \leq u_1^{-1/2}.
  \end{equation}
\end{corollary}

\vfill

\section{GPT-Like Large-Scale Language Models}
\label{sec:gpt}
\begin{figure}[H]
\centering \includegraphics[width=0.6\linewidth]{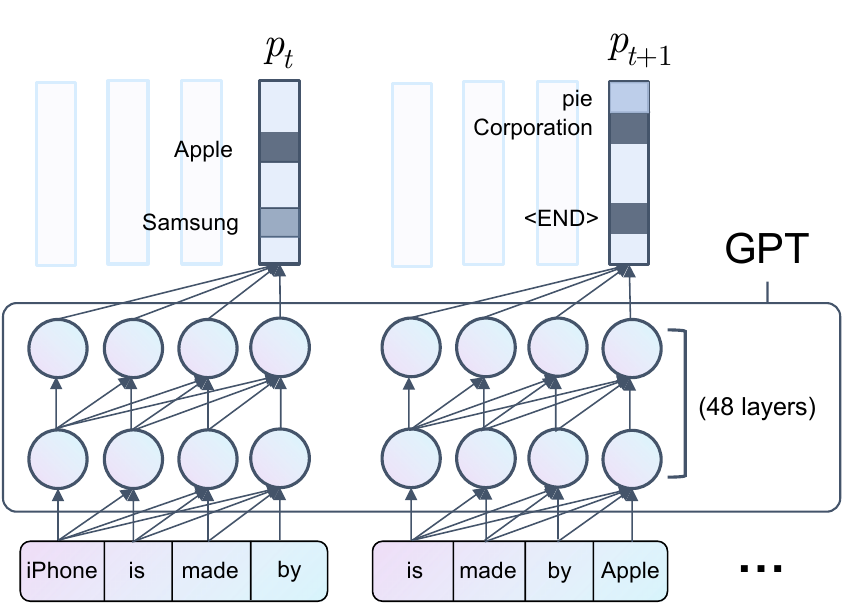}
\caption{Predicting the probability distribution $p_t$ over a vocabulary with the GPT-\texttt{xl} model, which has 48 layers.}
\label{fig:gpt}
\end{figure}

Generative pretrained transformers (GPTs) \citep{radford2019language, brown2020language} comprise a class of statistical language models that are implemented with neural networks. Large-scale GPT models, including the representative {\em ChatGPT} by OpenAI, have shown quasi-human-level performance in language understanding and question answering.

Figure \ref{fig:gpt} illustrates how $p_t$ was estimated in this work by using a GPT model, such as one with 48 neural-network layers (i.e., GPT2-\texttt{xl}). An input sequence of words is converted to vectors and then processed by multiple neural-network layers. To guarantee GPT's ``autoregressive'' nature, at every layer, the words visible during processing at timestep $t$ are limited to those previous to $t$.

A GPT model can process a variable-length sequence of words, subject
to a maximum length $c$ as given in \ref{eq:ctxlen} in the main
text. Figure \ref{fig:gpt} shows the case of $c=4$, in which a context
of four words is used for predicting $p_t$ at any timestep $t$. For
the work described in the main text, we used $c=512$ unless specified
otherwise.

LLMs are adept at estimating $p_t$, as they are trained with the objective of minimizing the statistical discrepancy between $p_t$ for the actual data and the model's estimate. More advanced and larger language models enhance the precision of $p_t$ estimation, bringing the measured correlation dimension closer to its true, universal value.

We used pretrained GPT-like models that are publicly available at \url{https://huggingface.co/models}. For the model size referred to as \texttt{xl} in the main text, we used \texttt{gpt2-xl} for English, \texttt{nlp-waseda/gpt2-xl-japanese} for Japanese, and \texttt{malteos/gpt2-xl-wechsel-german} for German. The tags \texttt{xl}, \texttt{large}, \texttt{medium}, and \texttt{small} correspond to GPT2 with 1.5 billion, 762 million, 345 million, and 117 million parameters, respectively.

For Chinese, we did not find a model that was consistent with the usual \texttt{xl} specification. However, we found a larger model for Chinese at \url{https://huggingface.co/IDEA-CCNL/Wenzhong2.0-GPT2-3.5B-chinese}, with around twice as many parameters as the usual \texttt{xl} specification. We also refer to this model as \texttt{xl}, as used in \ref{fig:corrdim}.

For even larger models, we considered the ``Yi'' family of GPT models provided at \url{https://huggingface.co/01-ai/Yi-34B}. Specifically, we used two models with 6 billion and 34 billion parameters. Among all publicly available models, the 34B Yi model is the state of the art for English, performing even better than several models with more parameters (e.g., Llama-2 70B). 

Models that are larger than \texttt{xl} often adopt half-precision floating-point numbers for their parameters. Accordingly, the numerical precision in evaluating $p_t$ is lower than with \texttt{xl} models. As a result, the dimension values may be more scattered, as was seen in \ref{fig:corrdim}. 

\vfill

\section{Dimension Reduction \label{sec:dimreduce}}

\begin{figure}[H]
\centering \includegraphics[width=0.4\linewidth]{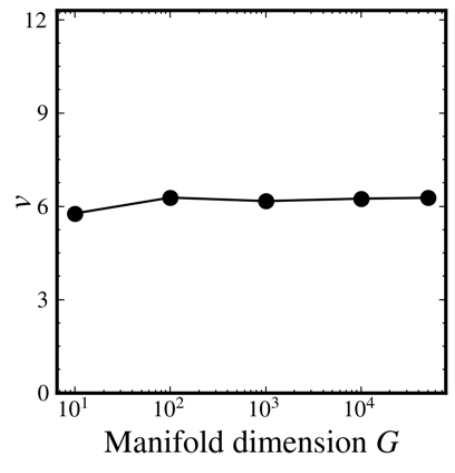}
\caption{Correlation dimension of {\em Don Quixote} as estimated using GPT2-\texttt{xl}, with respect to the statistical manifold's dimension $G$.}
\label{fig:donquixote-dim}
\end{figure}

We proposed a method to efficiently estimate the correlation dimension
$\hat{\nu}$. With this method, we have acquired dimension values that
are indistinguishable from those calculated with the direct, naive
method, which is 50X slower in our setting.

As detailed in \ref{eq:groupprob}, we mapped each distribution
$p_t$ into a new distribution $q_t$ with reduced dimensionality.
Consequently, with $q_t$, the correlation dimension can be estimated
at a lower computational cost using the following metric:
\begin{equation}
 \tilde{d}_\text{FR}(p_t, p_s)
 \equiv 2\arccos \left(\sum_{m=1}^M \sqrt{q_t(m) q_s(m)}\right).
 \label{eq:mod}
\end{equation}

Empirically, varying the manifold dimension $M$ does not alter the
results significantly. Figure \ref{fig:donquixote-dim} shows how the
estimated correlation dimension $\hat{\nu}$ evolves in relation to
$M$. The rightmost point represents the case of $M=|\vocab|$, i.e., no
reduction in the manifold's (topological) dimension. As $M$ decreases,
$\hat{\nu}$ barely changes until $M$ reaches its smallest value of
100. This suggests the reliability of the projection $p_t\mapsto q_t$
in preserving $\hat{\nu}$. In this letter, we employed this dimension
reduction method in the extensive large-scale experiments, as shown in
\ref{fig:corrdim}, where $M$ was set to 1000.

In Section \ref{sec:randomprocs}, this dimension reduction method was
also applied to acquire a low-dimensional visualization of a language
system.

\vfill
\newpage
\section{Local Fractality}
\label{sec:local}

\subsection{Comparison with Dirichlet Distribution}
\label{sec:local-dirichlet}

\begin{figure}[htbp]
\includegraphics[width=0.8\linewidth]{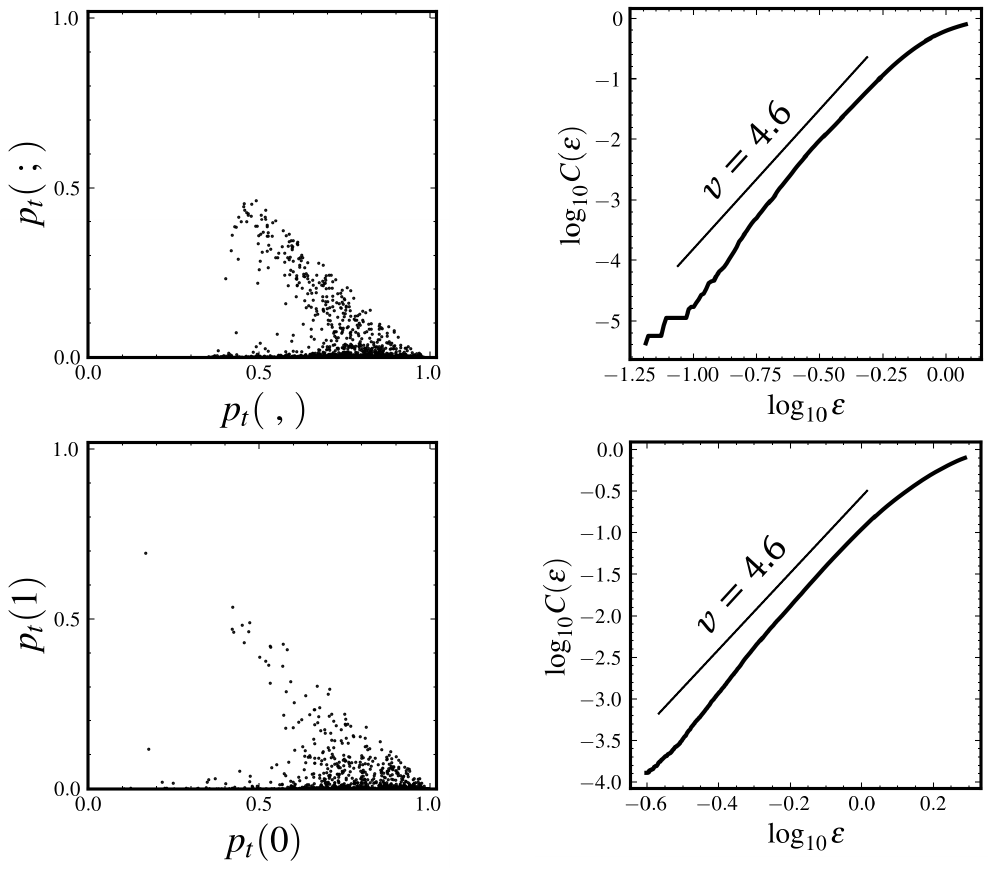}
\caption{Comparison between {\em Don Quixote} (top row) and a sequence of i.i.d. samples drawn from a Dirichlet distribution (bottom row). Left column: probabilities in two dimensions. Right column: correlation integral curves estimated using the points from the left column. It can be seen that the local fractal is reproduced well by the Dirichlet distribution.\label{fig:firstkind} }
\end{figure}

As mentioned in the main text, local fractals are observed in low-entropy regions. These fractals show a relatively simpler self-similar pattern that also appears in a Dirichlet distribution.

A Dirichlet distribution is a continuous probability distribution defined on a $K$-dimensional probability simplex:
\begin{equation}
  [z_1,\cdots,z_K]\subset[0,1]^K, ~~~~~\sum_{k=1}^K z_k=1,
\end{equation}
as specified by a parameter vector denoted as $\bm{\alpha}=[\alpha_1,\cdots,\alpha_K]\in\mathbb{R}^K$. The probability density function of a Dirichlet distribution over $[z_1,\cdots,z_K]$ is defined as follows:
\begin{equation}
  p(z_1,\cdots,z_K) = \frac{1}{B(\alpha)} \prod_{k=1}^K z_k^{\alpha_k - 1},
\end{equation}
where $B(\alpha)$ is the partition function.

We consider a Dirichlet distribution with $K=50257$, which is the same as the vocabulary size used by the English GPT2 models. We set the distribution's parameter vector, $\bm{\alpha}\in\mathbb{R}^{50257}$, such that $\alpha_1=3$, $\alpha_2=0.2$, and $\alpha_k=2.2\times 10^{-5}$ for $k=3,4,\cdots,50257$. 

Figure \ref{fig:firstkind} shows a comparison between {\em Don
Quixote} (upper) and i.i.d. samples drawn from the Dirichlet
distribution (lower), with restriction to the {\em low}-entropy region.
The local fractal shown in Figure \ref{fig:firstkind} (upper left) is
specific to the word ``,'': a point $p_t$ was selected to appear if
$H(p_t)<3$ and ``,'' has the largest probability in $p_t$ across the
vocabulary. The left-hand plots show the probability for each pair
of words, while the right-hand plots show the correlation integral
curves estimated from the points on the left. As seen in the figure,
the simple Dirichlet distribution approximates the real text well.
Therefore, the local fractal, which appears when the subsequent word
is almost determined, is reproducible with the Dirichlet distribution.

Note, however, that a Dirichlet distribution has completely different behavior with respect to global self-similarity. When the samples from a Dirichlet distribution are restricted to the {\em high}-entropy region, they do not exhibit the self-similar pattern of language. We demonstrate this in Section \ref{sec:sym-dirichlet}.

\newpage
\subsection{Local Fractals under Small Context Length}
\label{sec:local-ctxlen}

\begin{figure}[H]
\centering \includegraphics[width=\linewidth]{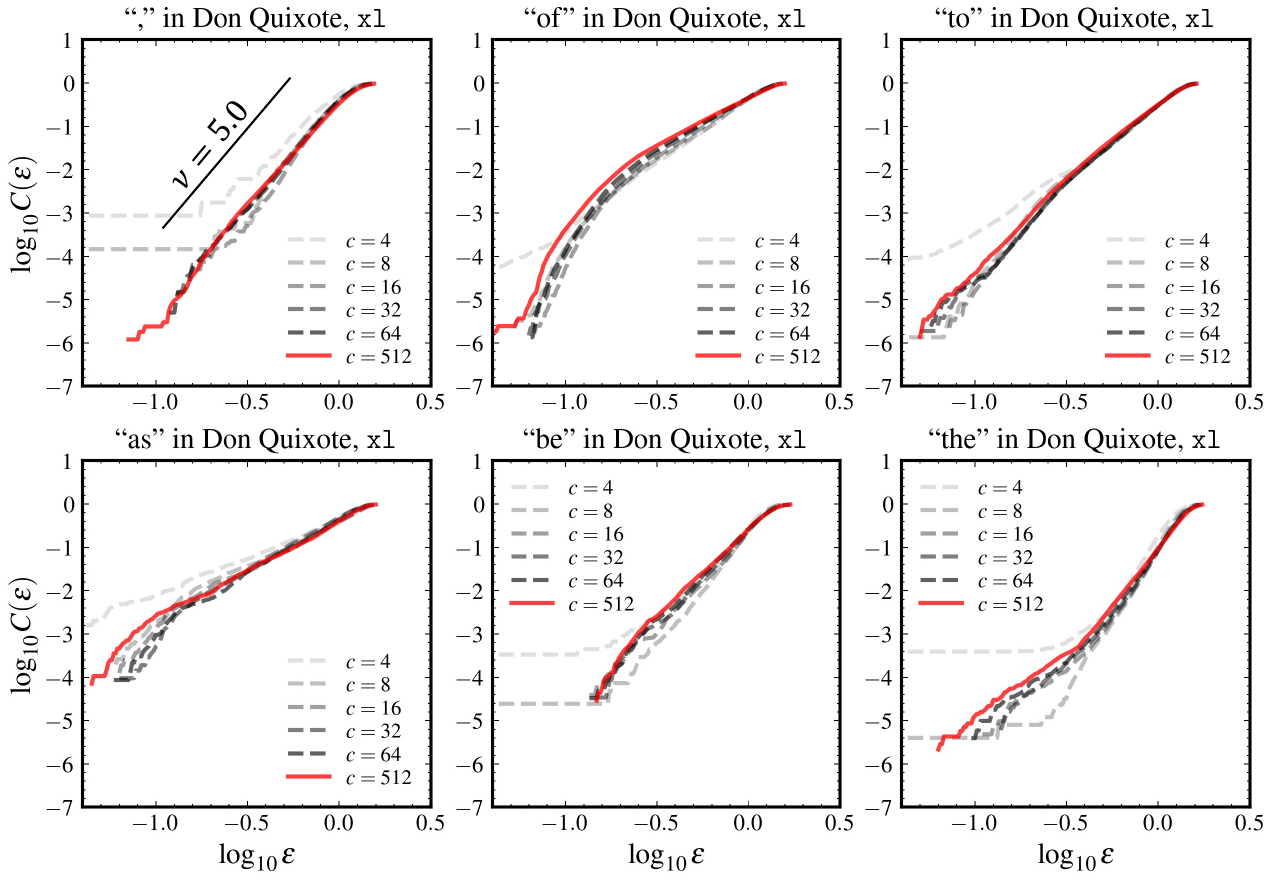}
\caption{Correlation integral curves for the local regions for
six frequent words with respect to different context length $c$,
estimated with GPT2-\texttt{xl} and a 100,000-word text segment in
Don Quixote. The local region for a word was identified by selecting
timesteps $t$ among $\{p_t\}$ at which the word has the largest
probability mass across the vocabulary and this probability mass
exceeds $\eta=0.5$. We were not able to examine $c<4$ because the
number of timesteps when $p_t$ falls in a local region decreased to
zero.}
\label{fig:donquixote-local}
\end{figure}

In our statistical manifold, each word corresponds to its own
low-entropy region where a local fractal may appear. Hence, local
fractals at different words may show varied dimensions. Here, we
examined several frequent words ``,'', ``of'', ``to'', ``as'', ``be'',
``the''. Furthermore, we examined the effect of $c$ on their local
patterns.

Figure \ref{fig:donquixote-local} displays the correlation integral
curves for specific local regions, with each subfigure corresponding
to a distinct word and each plot reflecting various $c$ values,
ranging from 512 to 4.

The local region associated with a word was identified by selecting
timesteps $t$ from the set $\{p_t\}$ where the word exhibits the
highest probability mass within the vocabulary, and this mass
surpasses a threshold of $\eta=0.5$. We refrained from examining $c$
values lower than 4 since, at $c<4$, the frequency of timesteps where
$p_t$ aligns with the local area was too scant (fewer than 50
instances), making it challenging to derive accurate estimates.

Figure \ref{fig:donquixote-local} reveals the {\em absence} of a
consistent shift in the correlation integral curve observed as $c$
decreases. The curves for smaller $c$ values show fluctuations around
the curve for $c=512$. While occasional upward shifts at $c=4$ can be
noted, these do not maintain consistency across different words.

This observation contrasts with Figure \ref{fig:donquixote}(c) of the
main text which displays a gradual shift of the global correlation
dimension with respect to $c$ is evident.

This contrast underlines a fundamental disparity between global and
local fractal phenomena as discussed in the text. Unlike the global
correlation dimension that characterizes the long-memory property of
language, the local dimensions characterize certain word-specific
characteristics that are invariant under a reduction of context
diversity (i.e., decreasing $c$).

We present a possible explanation for this invariance of the local
dimensions under varying $c$. Consider a specific word. A context is
a word sequence, and the word's occurring frequency is measured at
each location of the sequence, thus producing a mean $\mu$ and
variance $\sigma^2$ of the frequency distribution for this context.
Thus, a variation is expected across different contexts. Assuming
this variation abides with a scaling law: $\sigma^2\sim \mu^\gamma$,
where $\gamma$ is the scaling exponent and is somehow related to the
local correlation dimension. Reducing the context length $c$ is
equivalent to averaging multiple contexts and forces the LLM to
estimate the averaged occurring probability of the word.

The following theorem shows that $\gamma$ is
preserved under random pairing and merging of contexts.

\begin{theorem}(Invariance of $\gamma$ under context merging)
\label{thm:gamma-invariance}
Consider a set of independent contexts $\{\seqbm^1, \seqbm^2, ...,
\seqbm^{2L}, ...\}$ and a word $\word$.
Define $p(\word|\seqbm)$ as the occuring frequency of the word
$\word$ in a context $\seqbm$. If the variance and mean of the word's
frequency $p(\word|\seqbm)$ across these contexts follow
the scaling relationship $\text{Var}_l[p(\word|\seqbm^l)]
\propto \mathbb{E}_l[p(\word|\seqbm^l)]^\gamma$.
Then, upon merging these contexts pairwise into a new set
$\{\bar{\seqbm}\}$, where $\bar{\seqbm}^{l} = \{\seqbm_{2l-1};
\seqbm_{2l}\}$ and the frequency $p(\word|\bar{\seqbm}^l) =
(p(\word|\seqbm^{2l-1}) + p(\word|\seqbm^{2l}))/2$, the resulting
word frequency $p(\word|\bar{\seqbm}^l)$ still follows the same
scaling relationship with the exponent $\gamma$ remaining unchanged.
\end{theorem}

\begin{proof}
Define
$\mu :=\frac{1}{2L}\sum_{l=1}^{2L} p(\word|\seqbm^l)$ and
$\bar{\mu} := \frac{1}{L}\sum_{l=1}^{L} p(\word|\bar{\seqbm}^l)$.
Also, let
$\sigma^2 :=\frac{1}{2L}\sum_{l=1}^{2L} (p(\word|\seqbm^l) - \mu)^2$
and
$\bar{\sigma}^2:=\frac{1}{L}\sum_{l=1}^{L}(p(\word|\bar{\seqbm}^l)-\bar{\mu})^2$.
It is straightforward to show that $\bar{\mu}=\mu$. Furthermore,
\begin{align*}
\bar{\sigma}^2 &= \frac{1}{4L} \sum_{l=1}^L \left[ (p(\word|\seqbm^{2l-1}) - \mu) + (p(\word|\seqbm^{2l}) - \mu) \right]^2 \\
&= \frac{1}{2} \sigma^2 + \frac{1}{2L} \sum_{l=1}^L (p(\word|\seqbm^{2l-1})-\mu)(p(\word|\seqbm^{2l})-\mu),
\end{align*}
where the cross-term vanishes as $L\to\infty$ due to independence. Thus, $\bar{\sigma}^2 \to \sigma^2 / 2$.

Given $\sigma^2=\beta \mu^\gamma$, it follows that $\bar{\sigma}^2 =
\frac{\beta}{2} \bar{\mu}^\gamma$ at large $L$, indicating $\gamma$
remains unchanged under context merging.
\end{proof}

Topic models \citep{blei2003latent} have been successful in modeling
this variation of word frequency distribution across contexts by
considering the notion of topics. Interestingly,
\citet{doxas2010dimensionality} also showed that their observed
scaling structure in language sequences can be reproduced by topic
models.

\vfill
\section{Our Data}
\label{sec:data}

\begin{table}[htbp]
\centering
\caption{Summary of the datasets used in this letter.}
  \label{tbl:dataset}
  \begin{tabular*}{\linewidth}{@{\extracolsep{\fill}} lccr}
    \toprule
      Dataset & language & \# sequences & sequence length \\
    \midrule
    \multicolumn{4}{c}{\textbf{Books}} \\
      Gutenberg & English & 80 & 150,000 \\
                & Chinese & 32 & 150,000 \\
                & German & 16 & 150,000 \\
      Aozora-bunko & Japanese & 16 & 150,000 \\
      \midrule
      Stanford Encyclopedia of Philosophy & English & 60 & 25,000 \\
      Wikipedia webpages & English & 40 & 30,000 \\
      Academic papers & English & 242 & 15,000 \\
    \midrule
    \multicolumn{4}{c}{\textbf{Music}} \\
      \texttt{gtzan} & - & 1000 & 1,500 \\
    \bottomrule
  \end{tabular*}
\end{table}

Table \ref{tbl:dataset} summarizes the datasets used in this letter.

For books, we used texts from Project Gutenberg, except for the Japanese texts, which were taken from Aozora Bunko \footnote{\url{https://www.aozora.gr.jp/}}. A minimum file size of 1 megabyte was used as a threshold to obtain the longest texts in each collection. For each book, we skipped the first 50,000 words, because this first section includes catalog information and any prologue to the main text. The next 150,000 words of every text were thus used to estimate the correlation dimension.

The texts were split into words by using the tokenizers released with the GPT models. The tokenizers split a text into ``sub-word'' units, which is a common technique to deal with rare words within a fixed-size vocabulary. 

English texts from three other sources were also tested with our method: the Stanford Encyclopedia of Philosophy (SEP, \url{https://plato.stanford.edu/}), Wikipedia webpages, and academic papers (\url{https://huggingface.co/datasets/orieg/elsevier-oa-cc-by}). Texts beyond a certain length threshold were chosen: 25,000 words for SEP, 30,000 words for Wikipedia, and 15,000 words for the academic papers.

We also tested our method on the \texttt{gtzan} music dataset
\citep{tzanetakis2002musical} of WAV files. This dataset contains 1000
pieces of music (30 s each) categorized into 10 genres: rock, country,
metal, blues, disco, pop, hip-hop, jazz, reggae, and classical. Each
genre has 100 music pieces.

The \texttt{encodec} library was used to compress each music piece into a sequence of discrete codes constituting a vocabulary of size 2048; i.e., $|\vocab|=2048$. Each 30-s piece was encoded as a sequence of 1500 timesteps.

In general, the \texttt{encodec} library compresses a music piece into four tracks (sequences) based on four different codebooks (vocabularies), of which one is the main track. We considered only the main track because it contains the most information in a music piece.

\vfill
\newpage
\section{Supplementary Results on Book Texts}

\begin{figure}[H]
\begin{minipage}{0.50\linewidth}
\centering
\includegraphics[width=\linewidth]{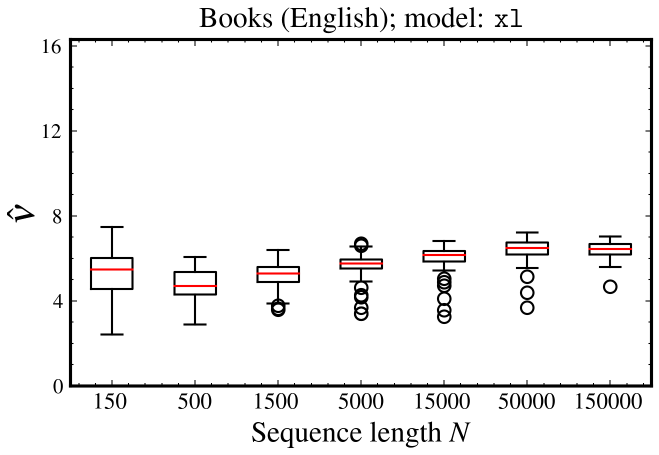}
\caption{Distribution of the correlation dimensions for the 80 English books in the Gutenberg Project, as measured using text fragments of different lengths.}
\label{fig:gutenberg-N}
\end{minipage}
\hspace{0.02\linewidth}
\begin{minipage}{0.41\linewidth}
\centering
\includegraphics[width=\linewidth]{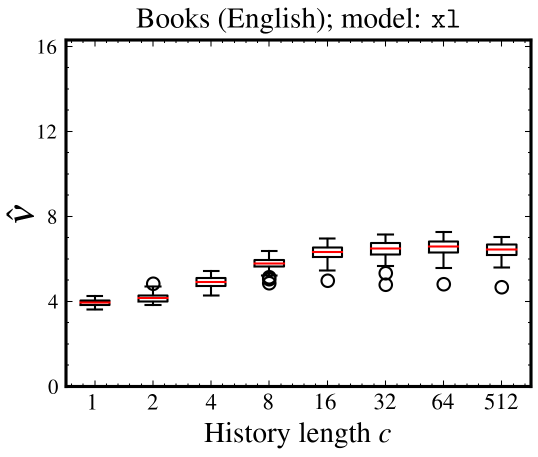}
\caption{Distribution of the correlation dimensions for the 80 English books in the Gutenberg Project with respect to $c$.}
\label{fig:gutenberg-ctxlen}
\end{minipage}
\end{figure}

In \ref{fig:donquixote} we reported the convergence of the correlation dimension with respect to the number of timesteps, $N$, and the context length $c$ for the book Don Quixote. Here, we provide the results for the entire Gutenberg corpus of 80 English books. Briefly, we observed the same convergence with respect to increasing $N$ or $c$ as in the case of Don Quixote.

\subsection{Effect of Sequence Length}
\label{sec:N}

Different values of $N$ represent different lengths of the text fragment taken from a book to estimate the correlation dimension. As explained above, we discarded the first 50,000 words to avoid catalog information. Then, starting from word 50,001, we took text fragments with different numbers of words. The longest fragment had 150,000 words. 

Figure \ref{fig:gutenberg-N} shows the distributions of the correlation dimensions (vertical axis) for the 80 books with $N$ ranging from 150 to 150,000. A clear convergence to a dimension around 6.5 is visible as $N$ increases. The correlation dimension's variance between books is large for small $N$ but decreases gradually as $N$ increases to 150,000. When $N=150000$, the correlation dimension is $6.39\pm 0.40$, as mentioned in the main text. 

\subsection{Effect of Context Length}
\label{sec:ctxlen}

The context length $c$ defined in \ref{eq:ctxlen} determines how many previous words are visible to the language model. For LLMs, $c$ cannot be infinitely large. We set $c=512$ for all textual data in this letter. 

It is of great interest how the context length $c$ affects the correlation dimension of natural language. A small $c$ value represents a short-memory approximation to natural language, which has been shown to have long-term memory. In particular, $c=1$ corresponds to a Markov process (forced by a real text).

Figure \ref{fig:gutenberg-ctxlen} shows the distribution of correlation dimensions for the 80 books with respect to increasing $c$. When $c=1$, the correlation dimensions are close to 4.0 with small variance. The dimension increases to around 6.5 when $c$ exceeds 16. At $c=512$, the correlation dimension is $6.39\pm 0.40$, as mentioned in the main text. This difference between the results with $c=1$ and $c=512$ further validates our finding in this letter that the self-similarity of language is related to the existence of long memory in text.

\vfill

\newpage
\section{Comparison with Random Processes}
\label{sec:randomprocs}
We tested three random processes in the statistical manifold $S$ of multinoulli distributions over $\vocab=\{1,2,\cdots,K\}$, a vocabulary of $K$ words.

\subsection{Uniform White Noise}
\label{sec:random}
\begin{figure}[htbp]
\centering \includegraphics[width=0.35\linewidth]{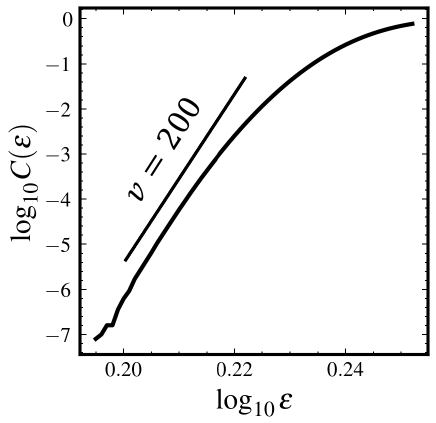}
\caption{Correlation integral curve of a uniform white-noise process in a statistical manifold. The process had $N=10000$ timesteps. $C(\varepsilon)$ was calculated using \ref{eq:corrintegral} and the Fisher-Rao distance in \ref{eq:fisher-rao}. }
\label{fig:uniform}
\end{figure}

The first random process was a uniform white-noise process. The manifold $(S, d_\text{FR})$ is isometric to the positive orthant of a hypersphere with coordinates $(\sqrt{\theta_1}, \sqrt{\theta_2}, \cdots, \sqrt{\theta_K})$. Therefore, a uniform white-noise process can be generated by uniformly drawing samples on this orthant, which can be done efficiently by normalizing Gaussian-distributed random vectors. We consider the uniform white-noise process in $S$ to be an analog of a Gaussian white-noise process in a Euclidean space, because each one is the ``entropy-maximizing'' distribution in its corresponding space.

As shown in Figure \ref{fig:uniform}, the correlation dimension of such a uniform white-noise process is large at over 100, which is identical to that of a Gaussian white-noise process in a Euclidean space.

\vfill

\subsection{Symmetric Dirichlet White Noise}
\label{sec:sym-dirichlet}

\begin{figure}[H]
\centering
\begin{minipage}{\linewidth}
\includegraphics[width=\linewidth]{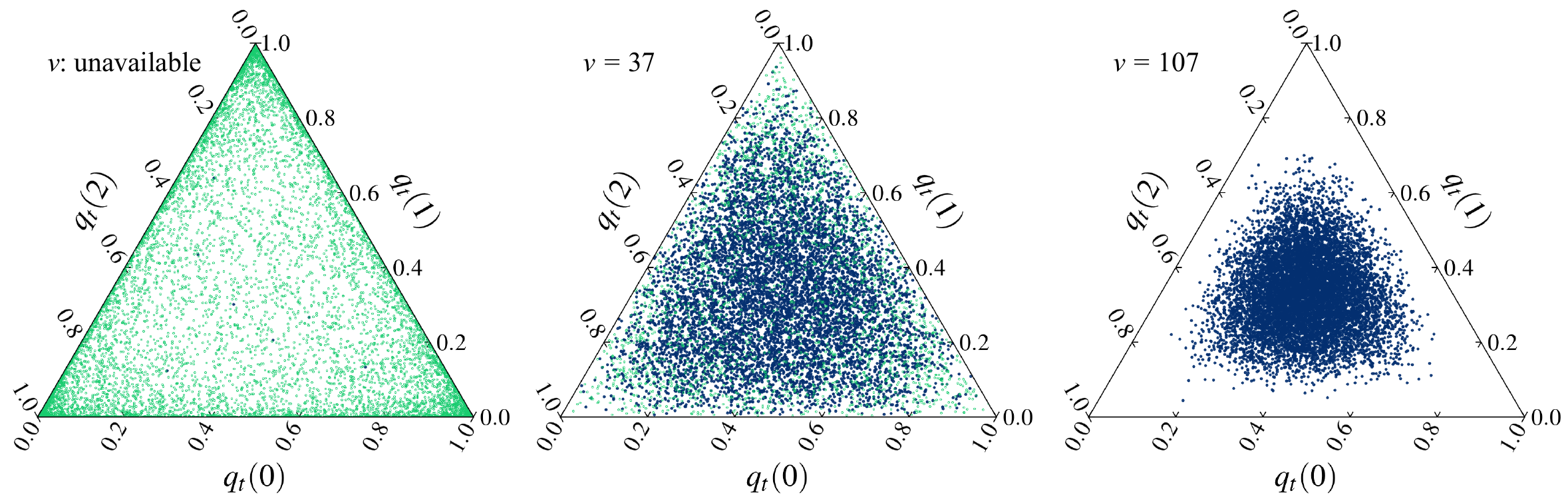}
\end{minipage}
\caption{Illustration of i.i.d. samples drawn from a symmetric
Dirichlet distribution with increasing parameter values of
$\alpha_k=1/K$ (left), $5/K$ (middle), and $20/K$ ($k=1,2,\cdots,K$).
The samples were probability vectors $p_t\in\mathbb{R}^{50257}$ and
were projected to $q_t\in\mathbb{R}^3$ via
\ref{eq:groupprob}, as detailed in Section \ref{sec:dimreduce}. The
points are blue for $H(p_t)\geq 3.0$ and green otherwise. The
estimated correlation dimension of the blue points is shown in each
plot's upper-left corner. For the left plot, the correlation dimension
is unavailable because there are almost no blue points.}
\label{fig:dirichlet-high-entropy}
\end{figure}

The second random process was a symmetric Dirichlet white-noise
process. In Section \ref{sec:local}, we showed that i.i.d. samples
from a Dirichlet distribution reproduce the local self-similarity of
language. A natural question is whether global self-similarity is also
reproducible by a Dirichlet distribution. To show that this is not the
case, we generated i.i.d. samples from several symmetric Dirichlet
distributions. For a symmetric Dirichlet distribution
$\text{Dir}(\bm{\alpha})$ ($\bm{\alpha}\in\mathbb{R}_+^K$), the
parameters $\alpha_k$ ($k=1,\cdots,K$) were set identically through
$k$. Three different symmetric Dirichlet distributions with parameters
$\alpha_k=1/K$, $5/K$, and $20/K$ ($k=1,2,\cdots,K$) were generated.
For consistency with the English GPT2 models, $K$ was set to 50257.
Time series of these Dirichlet distributions were produced as
$p_t\in\mathbb{R}^{50257}$ and then projected to $q_t\in\mathbb{R}^3$
via Formula (\ref{eq:groupprob}). Figure
\ref{fig:dirichlet-high-entropy} shows a plot map in these three
dimensions. The points are colored blue for $H(p_t)\geq 3.0$ and green
otherwise.

As seen in the figure, the correlation dimension depended on $\alpha_k$. For a larger $\alpha_k$, in the middle and right plots, the high-entropy blue points produced large correlation dimensions of 37 and 107, respectively, which indicates that these processes were barely self-similar. For a smaller $\alpha_k$ (left plot), all the points are green, and the correlation dimension of the high-entropy points is thus unavailable. These results show that simple i.i.d. samples from a Dirichlet distribution cannot reproduce the global fractals of natural language.

\subsection{Barab\'asi-Albert Network and A Fractional Variant}
\label{sec:barabasi-albert}

\begin{figure}[htbp]
\centering
\begin{minipage}{0.4\linewidth}
  \centering
  \includegraphics[width=\linewidth]{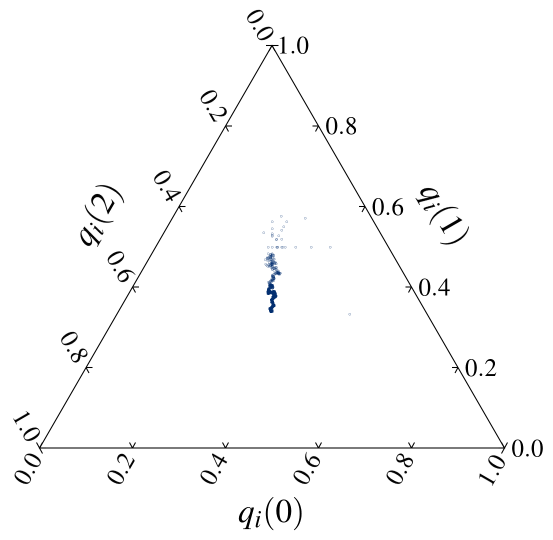}
  (a)
\end{minipage} \hspace{0.02\linewidth}
\begin{minipage}{0.4\linewidth}
  \centering
  \includegraphics[width=\linewidth]{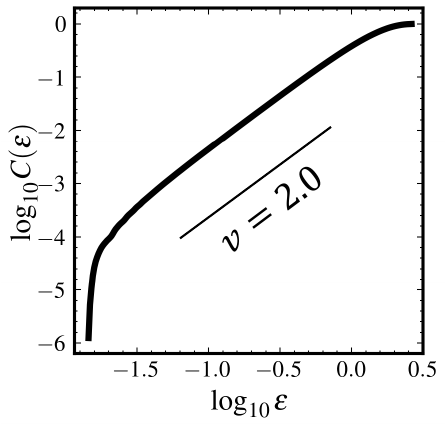}
  (b) 
\end{minipage}
\vskip 1em
\begin{minipage}{0.4\linewidth}
  \centering
  \includegraphics[width=\linewidth]{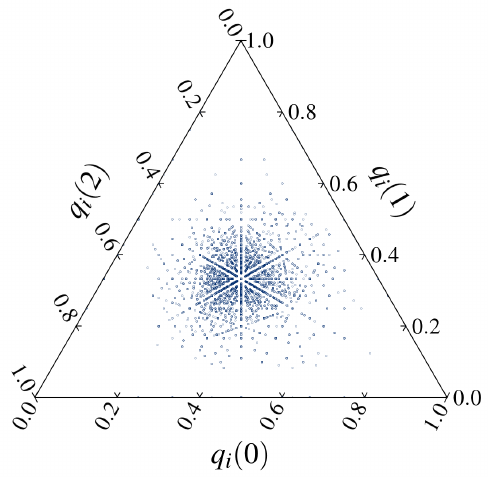}
  (c)
\end{minipage} \hspace{0.02\linewidth}
\begin{minipage}{0.4\linewidth}
  \centering
  \includegraphics[width=\linewidth]{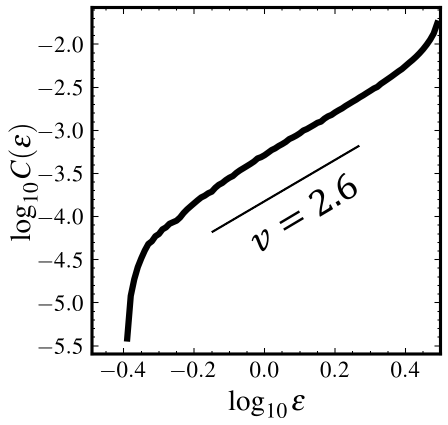}
  (d)
\end{minipage}
\vskip 1em
\begin{minipage}{0.4\linewidth}
  \centering
  \includegraphics[width=\linewidth]{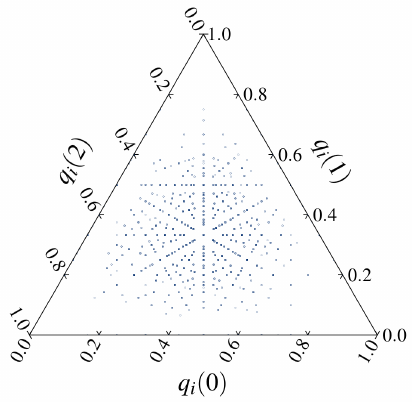}
  (e)
\end{minipage} \hspace{0.02\linewidth}
\begin{minipage}{0.4\linewidth}
  \centering
  \includegraphics[width=\linewidth]{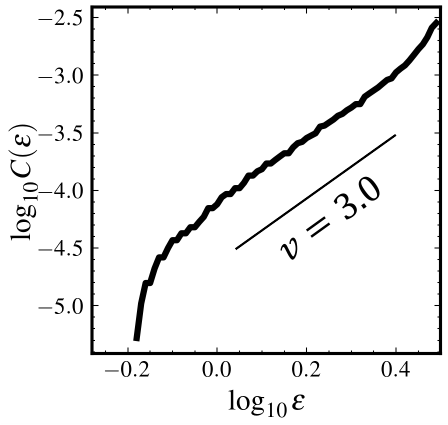}
  (f)
\end{minipage}
\caption{
Trajectory (left) and correlation integral curve (right) of (a-b) a
Barab\'asi-Albert model, (c-d) a FAPA model with $\kappa=0.005$, and
(e-f) a FAPA model with $\kappa=0.005$. The trajectories $\{q_t\}$ were
acquired from $\{p_t\}$ by the dimension-reducing mapping in
\ref{eq:groupprob}.
$m=1$ for all models.
}
\label{fig:barabasi-albert}
\end{figure}

The third random process was the generation process of a
Barab\'asi-Albert (BA) network \citep{barabasi1999emergence}. This
process is driven by a property called ``preferential attachment'' and
is a special case of the Simon model \citep{simon1955class}. A BA
network starts as a small network of $m_0$ nodes, where each node is
randomly connected to $m \geq m_0$ nodes. At each iteration, a new
node is added to the network and randomly connected to an existing
node in proportion to the nodes' (instantaneous) degrees. Consider the
generation process of a BA network with $K$ nodes, starting from $m_0$
nodes. For $K-m_0$ times, a new node is connected to an existing node,
following a multinoulli distribution over all $K$ nodes, including
those that have not yet been appended to the network at time $t$.

We take the connection probability as the ``next-word'' probability $p_t$. Therefore, at time $t$ ($t=1,2,\cdots,K-m_0$), the connection probability $p_t(k)$ is calculated as follows:
\begin{equation}
 p_t(k) = \left\{\begin{matrix}
 \dfrac{\text{deg}_t(k)}{\sum_{k'=1}^{m_0+t-1} \text{deg}_t(k')} &
 \text{if}~k \leq m_0+t-1 \\ 0 & \text{otherwise},
 \end{matrix}\right.
\label{eq:ba}
\end{equation}
where $\text{deg}_t(k)$ is the degree of node $k$ at time $t$.

Because the evolution of a BA network depends only on the nodes'
current numbers of connections, the connection probability
distribution $p_t$ identifies the state of the network as a dynamical
system.

In addition to the standard BA model, we considered a fractional
variant. This non-standard BA, called {\em fractional
anti-preferential attachment} (FAPA), is based on
\citet{rak2020fractional}. Unlike the standard BA, the probability
distribution $p_t$ over all existing node is truncated, and thus only
a fraction of the nodes are allowed to be connected to new nodes. In
a FAPA, the truncation is ``anti-preferential,'' that is, low-degree
nodes remain in this truncation while all high-degree nodes are
assigned probability mass zero.

The truncation rate is controlled by a ratio $\kappa$, and
Formula (\ref{eq:ba}) is modified as follows:
\begin{equation}
 p_t(k) = \left\{\begin{matrix}
 \dfrac{\text{deg}_t(k)}{\sum_{k'\in L_t} \text{deg}_t(k')} &
 \text{if}~k \in L_t \\
 0 & \text{otherwise},
 \end{matrix}\right.
 \label{eq:fapa}
\end{equation}
where $L_t$ is the set of nodes which are among the $|L_t|=\kappa
(m_0+t-1)$ nodes with the lowest degrees.

Figure \ref{fig:barabasi-albert} demonstrates the simulated processes
of the standard BA model (top row) and two FAPA models with
$\kappa=0.005$ (middle row) and $\kappa=0.002$ (bottom row),
respectively. All the models started from a single node ($m_0=1$); at
every iteration, one node is added to the network (i.e., $m=1$);
10,000 iterations were simulated. Hence, $p_t$ is a multinoulli
distribution over 10,000 classes (nodes).

The left-hand side of Figure \ref{fig:barabasi-albert} displays the
trajectories with the dimension reduced to 2D by using the
dimension-reduction mapping in
\ref{eq:groupprob}, and the right-hand side shows their
corresponding correlation integral curves calculated with
\ref{eq:corrintegral}.

A scaling structure can be seen in all three trajectories. The
correlation dimension was 2.0 for the standard BA (top row), 2.6 for
the FAPA with $\kappa=0.005$, and 3.0 for $\kappa=0.002$. The FAPA
model with the smaller $\kappa$ value showed a larger correlation
dimension.

The BA model is considered as a correspondent to the random walk model
in Euclidean spaces, and the FAPA model to the fractional Brownian
motion (fBm). The parameter $\kappa$ of the FAPA model is analogous to
the Hurst exponent of fBm models.

\vfill
\section{Using Euclidean Distance Metric}
\label{sec:euclid}

\begin{figure}[H]
\centering
  \centering
  \begin{minipage}{0.32\linewidth}
  \includegraphics[width=\linewidth]{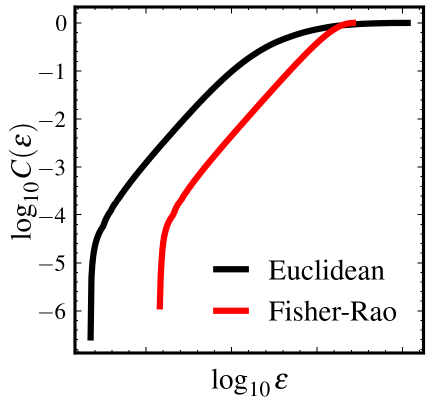}
  (a) BA model
  \end{minipage}
  \begin{minipage}{0.32\linewidth}
  \includegraphics[width=\linewidth]{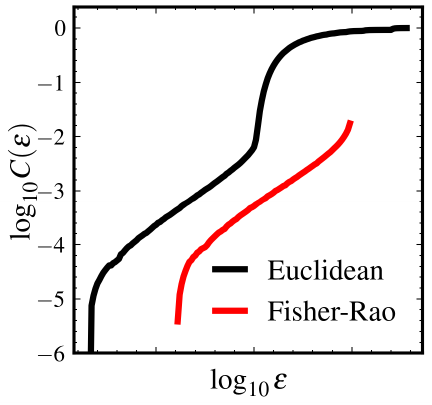}
  (b) FAPA ($\kappa=0.005$)
  \end{minipage}
  \begin{minipage}{0.32\linewidth}
  \includegraphics[width=\linewidth]{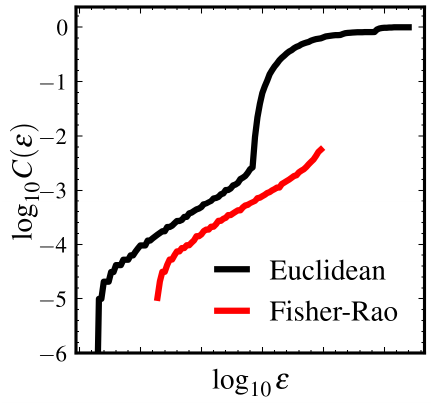}
  (c) FAPA ($\kappa=0.002$)
  \end{minipage}
  \\ \vskip 1em
  \begin{minipage}{0.32\linewidth}
  \includegraphics[width=\linewidth]{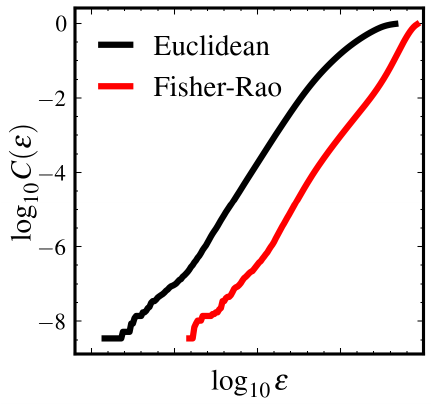}
  (d) {\em Don Quixote}
  \end{minipage}
  \begin{minipage}{0.32\linewidth}
  \includegraphics[width=\linewidth]{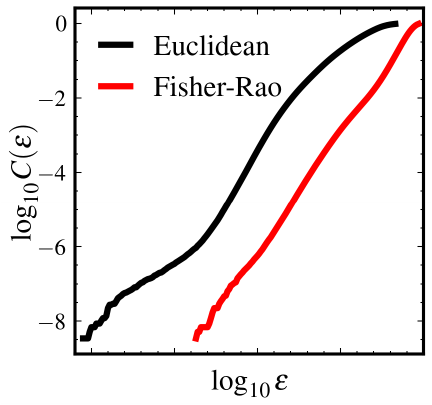}
  (e) {\scriptsize \em History of Modern Philosophy}
  \end{minipage}
  \begin{minipage}{0.32\linewidth}
  \includegraphics[width=\linewidth]{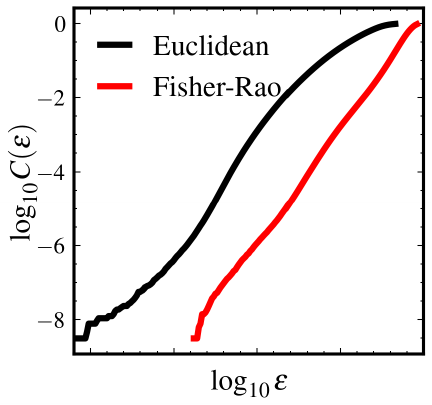}
  (f) {\em ``The Man Who Laughs''}
  \end{minipage}
\caption{
Correlation integral curves depicted using either the Fisher-Rao
distance metric (in red), or the conventional Euclidean distance
metric (in black). Figures (a-c) present comparisons for three models
outlined in Section \ref{sec:barabasi-albert} and Figure
\ref{fig:barabasi-albert}: (a) a standard BA model, (b) a FAPA model
with $\kappa=0.005$, and (c) a FAPA model with $\kappa=0.002$. Figures
(d-f) present curves for three literary sequences: (d) {\em Don
Quixote} by Miguel de Cervantes, (e) {\em History of Modern
Philosophy} by Richard Falckenberg, and (f) {\em The Man Who Laughs}
by Victor Hugo. The black curves have been shifted horizontally for a
clearer comparative illustration alongside the red curves. }
\label{fig:euclid-vs-fisher-rao}
\end{figure}

This section explores an alternative method for analyzing the
language sequences $\{p_t\}$ by measuring the distances between $p_t$
and $p_s$ ($\forall t,s$) using the conventional Euclidean distance
metric, $\lVert p_t-p_s\rVert$. This approach contrasts with the
Fisher-Rao distance method that is outlined in
\ref{eq:fisher-rao} and used throughout this paper.

Figure \ref{fig:euclid-vs-fisher-rao} illustrates the correlation
integral curves for six distinct sequences $\{p_t\}$. Each graph
depicts a single sequence, showing two curves: one generated by using
the Euclidean distance metric (in black) or the other using the
Fisher-Rao distance metric (in red). Figures
\ref{fig:euclid-vs-fisher-rao}(a-c) present simulated processes from
the three BA models discussed in Section \ref{sec:barabasi-albert}:
(a) the standard BA model, (b) a FAPA model with $\kappa=0.005$, and
(c) a FAPA model with $\kappa=0.002$. Figures
\ref{fig:euclid-vs-fisher-rao}(d-f) feature three language sequences
extracted from text segments of (d) {\em Don Quixote} by Miguel de
Cervantes (Gutenberg NO. 996), (e) {\em History of Modern Philosophy}
by Richard Falckenberg (Gutenberg NO. 11100), and (f) {\em The Man Who
Laughs} by Victor Hugo (Gutenberg NO. 12587).

Observations reveal that for the simple models depicted in (a-c), the
slopes obtained using Euclidean distances align with those from the
Fisher-Rao distance. In (b) and (c), a non-linear region emerges at
larger values of $\varepsilon$ when employing the Euclidean distance,
an effect not observed with the Fisher-Rao metric. However, for the
actual language sequences displayed in (d-f), the Euclidean-based
curves (in black) demonstrate compromised linearity compared with the
Fisher-Rao-based curves (in red), making it challenging to pinpoint
linear regions for curves in (e) and (f).

These findings highlight two critical insights. Firstly, there is a
notable consistency in the correlation dimension values between
statistical manifolds and traditional Euclidean spaces, especially for
well structured simple processes like the BA models. Thus, the
universally observed dimension value of 6.5 for natural language can
be analogously compared with random processes defined in Euclidean
spaces, such as fractional Brownian motion with a Hurst exponent of
$H=0.15$. This consistency, however, may be limited to $\{p_t\}$
sequences situated near the central region of the statistical
manifold---where the global fractal structure of language is located,
a condition met by the three BA models.

Secondly, the self-similarity property of language, as observed in the
global fractal, becomes significantly more pronounced and apparent
when analyzed within a statistical manifold utilizing the Fisher-Rao
distance metric, compared with a traditional English space. This
distinction is expected, given that the Euclidean distance between two
probability vectors lacks mathematical relevance, a point previously
discussed in the main text.

\vfill
\section{Correlation dimension of music data}
\label{sec:musicgen}

\begin{figure}[htbp]
\centering
\includegraphics[width=0.35\linewidth]{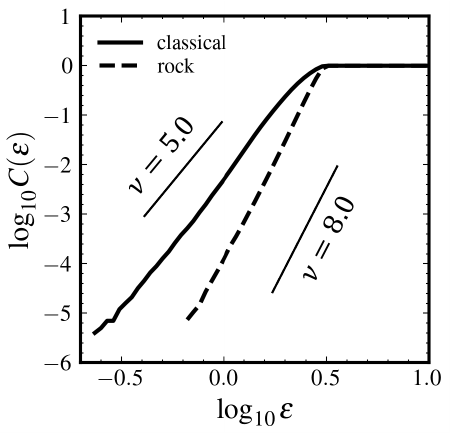} \\
\vskip 1em
\includegraphics[width=0.75\linewidth]{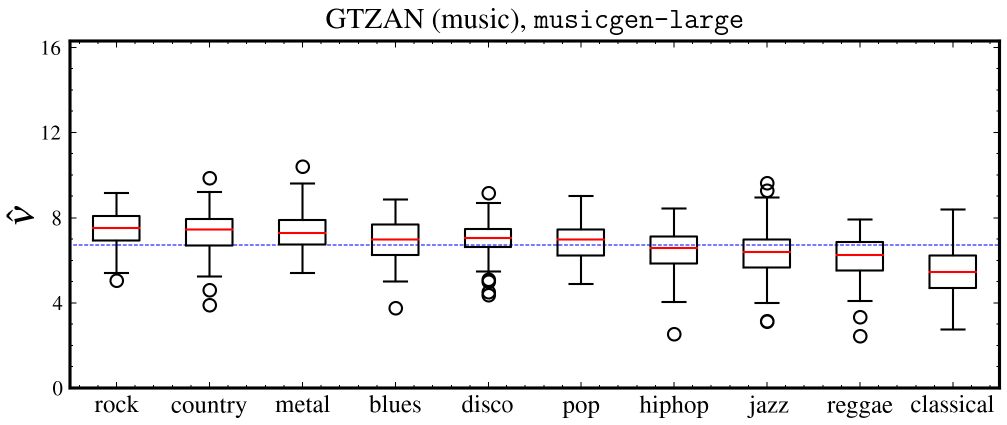}
\caption{ (Upper) Correlation integral curves for two music pieces categorized as ``classical'' and ``rock'' in the \texttt{GTZAN} dataset. (Lower) Correlation dimensions of compressed music pieces categorized by genre. The dashed blue line indicates the average over all genres.}
\label{fig:music-genre}
\end{figure}

As mentioned in Section \ref{sec:data}, each music piece was encoded as a sequence of discrete codes from a ``vocabulary'' of 2048 ``words.'' Then, we acquired the system states $p_t$ for each timestep $t$ ($t=1,\cdots,N$) by using an LLM trained for music generation. In this letter, we used the \texttt{musicgen} model \citep{copet2023simple} of size \texttt{large} (with $3\times 10^9$ parameters). The \texttt{musicgen} model works in the same way as GPT2. Every music piece in the \texttt{GTZAN} dataset was compressed into a discrete sequence of 1,500 timesteps, i.e., $N=1500$.

The context length $c$ defined in \ref{eq:ctxlen} was not limited: we used the whole context without approximation. Unlike for texts, which have tens of thousands of timesteps, the music sequences only had 1,500 timesteps, which is within the \texttt{musicgen} model's limitation on the context length.

In estimating the correlation dimension $\hat{\nu}$ for the music data, we excluded the low-entropy timesteps from $\{p_t\}$. The maximum-probability threshold $\eta$ was set to 0.5, the same value used for characterizing texts. 

Figure \ref{fig:music-genre} (upper) shows the correlation integral curves for two typical music pieces in the dataset. Linear scaling is clearly visible throughout the whole regions for both pieces. The music piece categorized as ``classical'' had a correlation dimension (i.e., the slope) around 5.0, while the piece categorized as ``rock'' had a dimension of 8.0. Next, Figure \ref{fig:music-genre} (lower) shows the distribution of the correlation dimensions of all 1000 pieces, grouped by genre. Rock music had the highest correlation dimension on average, while classical music had the lowest.

This discrepancy in the correlation dimension between music genres aligns well with our perception that classical music is less random than rock or metal music. Nevertheless, as music pieces may show large variety in terms of many other factors (e.g., timbre), we are less likely to observe a universal value of the correlation dimension for music than for language.

\end{document}